\documentclass{article} 
\usepackage{nips12submit_e,times}

\title{The Perturbed Variation}

\author{
Maayan Harel \\
Department of Electrical Engineering\\
Technion, Haifa, Israel\\
\texttt{maayanga@tx.technion.ac.il} \\
\And
Shie Mannor \\
Department of Electrical Engineering \\
Technion, Haifa, Israel\\
\texttt{shie@ee.technion.ac.il} \\
}

\usepackage{tikz}
\usepackage{pgfplots}
\pgfplotsset{every axis/.append style={thick}}
\usepackage{float}

\newfloat{floatbox}{thp}{lob}[section]
\floatname{floatbox}{Figure}
\usepackage{graphicx}
\usepackage{subfigure}
\usepackage{multirow}

\usepackage[numbers,sort&compress]{natbib}

\usepackage{algorithm}
\usepackage{algorithmic}

\usepackage{amsthm}
\usepackage{amsmath} 
\usepackage{amssymb} 

\usepackage[hyphens]{url}
\usepackage{hyperref}

\usepackage{color}

\nipsfinalcopy 

\begin{document}

\makeatletter
\theoremstyle{plain}
\newtheorem*{rep@theorem}{\rep@title}
\newcommand{\newreptheorem}[2]{%
\newenvironment{rep#1}[1]{%
 \def\rep@title{#2 \ref{##1}}%
 \begin{rep@theorem}}%
 {\end{rep@theorem}}}
\makeatother

\newtheorem{theorem}{Theorem}
\newreptheorem{theorem}{Theorem}
\newtheorem{Proof}{Proof}
\newtheorem{lemma}[theorem]{Lemma}
\newtheorem{proposition}[theorem]{Proposition}
\newtheorem{corollary}[theorem]{Corollary}
\newtheorem{definition}[theorem]{Definition}
\newtheorem{assumption}{Assumption}
\newtheorem{remark}{Remark}
\newtheorem{procedure}{Procedure}

\newcommand{\PV}{\text{PV}}
\newcommand{\PVh}{\widehat{\text{PV}}}
\newcommand{\PPV}{\widehat{\text{PPV}}}
\newcommand{\bd}{\bf{d}}
\newcommand{\brck}[1]{\ensuremath{\left( #1 \right)}}

\maketitle

\begin{abstract}
We introduce a new discrepancy score between two distributions that gives an indication on their \emph{similarity}.
While much research has been done to determine if two samples come from exactly the same
distribution, much less research considered the problem of determining if two finite samples come from similar distributions.
The new score gives an intuitive interpretation of similarity;  it optimally perturbs the distributions so that they best fit each other. The score is defined between distributions, and can be efficiently estimated from samples. We provide convergence bounds of the estimated score, and develop hypothesis testing procedures that test if two data sets come from similar distributions. The statistical power of this procedures is presented in simulations. We also compare the score's capacity to detect similarity with that of other known measures on real data.  
\end{abstract}

\section{Introduction \label{sec_Intro}}

The question of similarity between two sets of examples is common to many fields, including statistics, data mining, machine learning and computer vision. 
For example, in machine learning, a standard assumption is that the training and test data are generated from the same distribution. However, in some scenarios, such as Domain Adaptation (DA), this is not the case and the distributions are only assumed similar. It is quite intuitive to denote when two inputs are similar in nature, yet the following question remains open: given two sets of examples, how do we test whether or not they were generated by similar distributions?    
The main focus of this work is providing a similarity score and a corresponding statistical procedure that gives one possible answer to this question.

Discrepancy between distributions has been studied for decades, and a wide variety of distance scores have been proposed. However, not all proposed scores can be used for testing similarity. The main difficulty is that most scores have not been designed for statistical testing of similarity but equality, known as the Two-Sample Problem (TSP). Formally, let $P$ and $Q$ be the generating distributions of the data; the TSP tests the null hypothesis $H_0:P=Q$ against the general alternative $H_1:P\neq Q$. This is one of the classical problems in statistics. However, sometimes, like in DA, the interesting question is with regards to similarity rather than equality. By design, most equality tests may not be transformed to test similarity; see Section \ref{sec_relatedWork} for a review of representative works.   

In this work, we quantify similarity using a new score, the Perturbed Variation (PV).
We propose that similarity is related to some predefined value of permitted variations. Consider the gait of two male subjects as an example. If their physical characteristics are similar, we expect their walk to be similar, and thus assume the examples representing the two are from similar distributions.  This intuition applies when the distribution of our measurements only endures small changes for people with similar characteristics. Put more generally, similarity depends on what ``small changes'' are in a given application, and implies that similarity is domain specific. The PV, as hinted by its name, measures the discrepancy between two distributions while allowing for some perturbation of each distribution; that is, it allows small differences between the distributions. What accounts for small differences is a parameter of the PV, and may be defined by the user with regard to a specific domain. Figure \ref{fig:PVillustration} illustrates the PV. 
Note that, like perceptual similarity, the PV turns a blind eye to variations of some rate.

\begin{figure}[t]
  \begin{center}
    \includegraphics[ trim =1cm 3.5cm 0cm 2cm ,clip, width=0.33\textwidth] {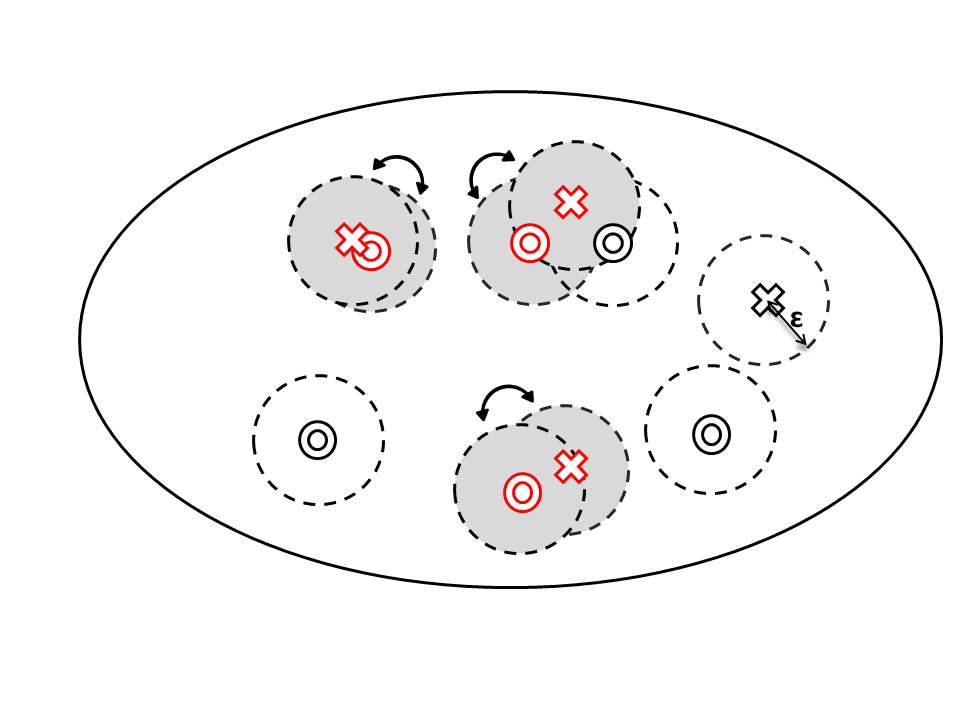}
  \end{center}
\vspace{-0.3cm}
\caption{ \small{X and O identify samples from two distributions, doted circles denote allowed perturbations. Samples marked in red are matched with neighbors, while the unmatched samples indicate the PV discrepancy.  }\label{fig:PVillustration}}
\end{figure}

\section{The Perturbed Variation \label{Sec_PV}}

The PV on continuous distributions is defined as follows:
\begin{definition}
Let $P$ and $Q$ be two distributions  on a Banach space $\mathcal{X}$, and let $M(P,Q)$ be the set of all joint distributions on $\mathcal{X}\times \mathcal{X}$ with mariginals $P$ and $Q$.
The PV, with respect to a distance function $\bd :\mathcal{X}\times \mathcal{X} \rightarrow \mathbb{R}$ and $\epsilon$, is defined by
\begin{align}\label{PV}
 \PV(P,Q,\epsilon,\bd) \doteq &\inf_{\mu\in M(P,Q)}  \mathbb{P}_{\mu}[d(X,Y)>\epsilon] ,
\end{align}
over all pairs $(X,Y)\sim \mu $, such that the marginal of $X$ is $P$ and the marginal of $Y$ is $Q$.
\end{definition}

Put into words, Equation (\ref{PV}) defines the joint distribution $\mu$ that couples the two distributions such that the probability of the event of a pair $(X,Y)\sim \mu $ being within a distance grater than $\epsilon$ is minimized.   

The solution to (\ref{PV}) is a special case of the classical mass transport problem of Monge \cite{Monge1781} and its version by Kantorovich:
 $\inf_{\mu\in M(P,Q)} \int_{\mathcal{X}\times \mathcal{X}} c(x,y)d\mu(x,y),$
where $c:\mathcal{X}\times\mathcal{X} \rightarrow \mathbb{R} $ is a measurable cost function.
When $c$ is a metric, the problem describes the $1^{st}$ Wasserstein metric.
Problem (\ref{PV}) may be rephrased as the optimal mass transport problem with the cost function $c(x,y) = 1_{[d(x,y)>\epsilon]}$,
and may be rewritten as  $\inf_{\mu} \iint 1_{[d(x,y)>\epsilon]}\mu(y|x)dy\, P(x)dx$. The probability $\mu(y|x)$ defines the transportation plan of $x$ to $y$. The PV optimal transportation plan is obtained by perturbing the mass of each point $x$ in its $\epsilon$ neighborhood  so that it redistributes to the distribution of $Q$. These small perturbations do not add any cost, while transportation of mass to further areas is equally costly. Note that when $P=Q$ the PV is zero as the optimal plan is simply the identity mapping.  
Due to its cost function, the PV it is not a metric, as it is symmetric but does not comply with the triangle inequality and may be zero for distributions $P\neq Q$. 
Despite this limitation, this cost function fully quantifies the intuition that small variations should not be penalized when similarity is considered.   
In this sense, similarity is not unique by definition, as more than one distribution can be similar to a reference distribution.

The PV is also closely related to the Total Variation distance  (TV) that may be written, using a coupling characterization, as
$TV(P,Q) = \inf_{\mu\in M(P,Q)} \mathbb{P}_{\mu}\left[X \neq Y \right]$  \cite{ruschendorf2007monge}.
This formulation argues that any transportation plan, even to a close neighbor, is costly. Due to this property, the TV is known to be an overly sensitive measure that overestimates the distance between distributions. For example, consider two distributions defined by the dirac delta functions $\delta(a)$ and $\delta(a+\epsilon)$. For any $\epsilon$, the TV between the two distributions is 1, while they are intuitively similar.  The PV resolves this problem by adding perturbations, and therefore is a natural extension of the TV. Notice, however, that the $\epsilon$ used to compute the PV need not be infinitesimal, and is defined by the user.

The PV can be seen as a conciliatory between the Wasserstein distance and the TV. As explained, it relaxes the sensitivity of the TV; however, it does not ``over optimize" the transportation plan. Specifically, distances larger than the allowed perturbation are discarded. This aspect also contributes to the efficiency of estimation of the PV from samples; see Section \ref{estimationPV}.

\subsection{ The Perturbed Variation on Discrete Distributions}
It can be shown that for two discrete distributions Problem (\ref{PV}) is equivalent to the following problem.
\begin{definition}
Let $\mu_1$ and $\mu_2$ be two discrete distributions on the unified support $\{a_1,...,a_N\}$. Define the neighborhood of $a_i$ as $\text{ng}(a_i,\epsilon) = \{ z \, ; d(z,a_i) \leq \epsilon \}$.
The $\PV(\mu_1,\mu_2,\epsilon, \bd )$ between the two distributions is:
\vspace{-0.4cm}
\begin{align}\label{PVdis}
& \min _{ w_i\geq 0,v_i\geq 0,Z_{ij}\geq 0}  \frac{1}{2}\sum_{i=1}^{N}w_i + \frac{1}{2}\sum_{j=1}^{N}v_j \\
  \text{s.t.} & \sum_{a_j\in \text{ng}(a_i,\epsilon)} Z_{ij}+w_i = \mu_1(a_i) , \small{\forall i}  \nonumber \\
& \sum_{a_i\in \text{ng}(a_j,\epsilon)} Z_{ij}+v_j   = \mu_2(a_j) ,  \small{\forall j}  \nonumber \\
& \quad Z_{ij} = 0  \, , \quad \small{\forall(i,j)\not\in\text{ng}(a_i,\epsilon)}. \nonumber 
\end{align}
\vspace{-0.5cm}
\end{definition}
Each row in the matrix $Z\in \mathbb{R}^{N\times N}$ corresponds to a point mass in $\mu_1$, and each column to a point mass in $\mu_2$. For each $i$,  $Z(i,:)$ is zero in columns corresponding to non neighboring elements, and non-zero only for columns $j$  for which transportation between $\mu_2(a_j)\rightarrow \mu_1(a_i)$ is performed. The discrepancies between the distributions are depicted by the scalars $w_i$ and $v_i$ that count the ``leftover" mass in $\mu_1(a_i)$ and $\mu_2(a_j)$. The objective is to minimize these discrepancies, therefore matrix $Z$ describes the optimal transportation plan constrained to $\epsilon$-perturbations. An example of an optimal plan is presented in Figure \ref{PV_example}.
\begin{floatbox}[t]
\label{PV_example}
\centering
\fbox{
\begin{minipage}[t]{0.33\textwidth}
\small{
$ PV(\mu_1,\mu_2,\epsilon) = \frac{1}{2}$: \\
\vspace{0.05cm}

$a_1=0,a_2=1,a_3=2,a_4=2.1$\\
$w_1 = w_2 = \frac{1}{4},w_3=w_4=0$\\
$v_4 = \frac{1}{2} ,v_1=v_2=v_3=0$ 

$ Z =  \left[ \begin{array}{cccc}
0 & 0 & 0 & 0 \\
0 & \frac{1}{4} & 0 & 0 \\
0 & 0 & 0 & \frac{1}{4} \\
0 & 0 & 0 & 0 \\
 \end{array} 
\right]$
}
\end{minipage}\hspace{1cm}
\begin{minipage}[t]{0.35\textwidth}
\vspace{1pt}
\begin{tikzpicture}[scale=0.55]
\begin{axis}[
xmin=-0.1,xmax=3.5,
ymin=0,ymax=1,
enlargelimits=false,
xtick=data,
ytick={0.25,0.5,0.75},
]
\addplot+[ycomb,line width=2pt] plot coordinates
{(0,0.25) (1,0.5) (2,0.25)};
\addplot+[ycomb,line width=2pt] plot coordinates
{(1,0.25) (2.1,0.75)};
\legend{$\mu_1$,$\mu_2$}
\end{axis}
\node at (4,4.8) {$\epsilon$};
\draw [<->] (3.8,4.5) -- (4.2,4.5);
\node at (1.1,0.9) {$\geq \epsilon$};
\draw [<->] (0.25,0.5) -- (2,0.5);
\end{tikzpicture}
\end{minipage}}
\caption{Illustration of the PV score between discrete distributions.}
\end{floatbox} 

\subsection{Estimation of the Perturbed Variation \label{estimationPV}}
Typically, we are given samples from which we would like to estimate the PV.
Given two samples $S_1 =\{x_1,...,x_n\}$ and $S_2 =\{y_1,...,y_m\}$,  generated by distributions $P$ and $Q$ respectively, $\PVh(S_1,S_2,\epsilon,d)$ is:
\vspace{-0.5cm}
\begin{align}\label{PVhat}
& \min_{w_i\geq 0,v_i\geq 0,Z_{ij}\geq 0}  \frac{1}{2n}\sum_{i=1}^{n}w_i + \frac{1}{2m}\sum_{j=1}^{m}v_j \\
\text{s.t.} &  \sum_{y_j\in \text{ng}(x_i,\epsilon)} Z_{ij}+w_i = 1 , \, \sum_{x_i\in \text{ng}(y_j,\epsilon)} Z_{ij}+v_j=1,  \quad  \forall i,j  \nonumber \\
& \quad Z_{ij} = 0  \, , \quad\quad\,   \forall(i,j)\not\in\text{ng}(x_i,\epsilon), \nonumber
\end{align}
where $Z\in\mathbb{R}^{n\times m}$.
When $n=m$, the optimization in (\ref{PVhat}) is identical to (\ref{PVdis}), as in this case the samples define a discrete distribution. However, when $n\neq m$ Problem (\ref{PVhat}) also accounts for the difference in the size of the two samples.

Problem (\ref{PVhat}) is a linear program with constraints that may be written as a totally unimodular matrix. It follows that one of the optimal solutions of (\ref{PVhat}) is integral \cite{schrijver1998theory}; that is, the mass of each sample is transferred as a whole.
This solution may be found by solving the optimal assignment on an appropriate bipartite graph \cite{schrijver1998theory}.
Let $G =(V=(A,B),E)$ define this graph, with $A = \{x_i, w_i \, ; i=1,...,n\}$ and $B = \{y_j, v_j \, ; j=1,...,m\}$ as its bipartite partition. The vertices $x_i\in A$ are linked with edge weight zero to $y_j\in \text{ng}(x_i)$ and with weight $\infty$ to $y_j\not\in \text{ng}(x_i)$. In addition, every vertex $x_i$ ($y_j$) is linked with weight 1 to $w_i$ ($v_j$). To make the graph complete, assign zero cost edges between all vertices $x_i$ and $w_k$ for $k\neq i$ (and vertices $y_j$ and $v_k$ for $k\neq j$).

We note that the Earth Mover Distance (EMD) \cite{rubner1998metric}, a sampled version of the transportation problem, is also formulated by a linear program that may be solved by optimal assignment. For the EMD and other typical assignment problems, the computational complexity is more demanding, for example using the Hungarian algorithm it has an $O(N^3)$ complexity, where $N=n+m$ is the number of vertices \cite{Ahuja93}. Contrarily, graph $G$, which describes $\PVh$, is a simple bipartite graph for which maximum cardinality matching, a much simpler problem, can be applied to find the optimal assignment.
To find the optimal assignment, first solve the maximum matching on the partial graph between vertices $x_i,y_j$ that have zero weight edges (corresponding to neighboring vertices). Then, assign vertices $x_i$ and $y_j$ for whom a match was not found with $w_i$ and $v_j$ respectively; see Algorithm \ref{Alg_PVhat} and Figure \ref{fig:PVillustration} for an illustration of a matching. It is easy to see that the solution obtained solves the assignment problem associated with $\PVh$.

\begin{algorithm}[t]
\caption{\emph{Compute $\PVh(S_1,S_2,\epsilon,{\bf d})$}}
\label{Alg_PVhat}
\begin{algorithmic}
\vspace{0.3cm}
\STATE \textbf{Input:} $S_1 = \{x_1,...,x_n\}$ and $S_2=\{y_1,...,y_m\}$, $\epsilon$ rate,  and distance measure ${\bf d}$.
\STATE  1. Define $\hat{G} =(\hat{V}=(\hat{A},\hat{B}),\hat{E})$: $\,\, \hat{A}= \{ x_i \in S_1 \}, \, \hat{B}= \{ y_j \in S_2 \},$
\STATE  $\quad$ Connect an edge $e_{ij}\in \hat{E}$  if $d(x_i,y_j)\leq \epsilon .$
\STATE 2. Compute the maximum matching on   $\hat{G}$.
\STATE 3. Define $S_w$ and $S_v$ as number of unmatched edges in sets $S_1$ and $S_2$ respectively.
\STATE \textbf{Output:} $\widehat{PV}(S_1,S_2,\epsilon,d) = \frac{1}{2}(\frac{S_w}{n}+\frac{S_v}{m}$).
\vspace{0.1cm}
\end{algorithmic}
\end{algorithm}

The complexity of Algorithm \ref{Alg_PVhat} amounts to the complexity of the maximal matching step and of setting up the graph, i.e., additional $O(nm)$ complexity of computing distances between all points.  Let $k$ be the average number of neighbors of a sample, then the average number of edges in the bipartite graph $\hat{G}$ is $|\hat{E}|=n\times k$. The maximal cardinality matching of this graph is obtained in $O(kn\sqrt{(n+m)})$ steps, in the worst case \cite{Ahuja93}.

\section{Related Work \label{sec_relatedWork}}
Many scores have been defined for testing discrepancy between distributions. We focus on representative works for nonparametric tests that are most related to our work.   
First, we consider statistics for the Two Sample Problem (TSP), i.e., equality testing, that are based on the asymptotic distribution of the statistic conditioned on the equality. Among these tests is the well known  Kolmogorov-Smirnov test (for one dimensional distributions), and its generalization to higher dimensions by minimal spanning trees \cite{friedman1979multivariate}. A different statistic is defined by the portion of k-nearest neighbors of each sample that belongs to different distributions; larger portions mean the distributions are closer \cite{schilling1986multivariate}. These scores are well known in the statistical literature but cannot be easily changed to test similarity, as their analysis relies on testing equality.  

As discussed earlier, the $1^{st}$ Wasserstein metric and the TV metric have some relation to the PV. The EMD and histogram based $L_1$ distance are the sample based estimates of these metrics respectively. In both cases, the distance is not estimated directly on the samples, but on a higher level partition of the space: histogram bins or signatures (cluster centers). As a result, these estimators have inaccuracies.  
Contrarily, the PV is estimated directly on the samples and converges to its value between the underlying continuous distributions.  
We note that after a good choice of signatures, the EMD captures perceptual similarity, similar to that of the PV.    
However, due to the abstraction to signatures, the EMD does not converge to the Wasserstein metric between the continuous distributions, and therefore is commonly used to rate distances and not for statistical testing.
It is possible to consider the PV as a refinement of the EMD notion of similarity; instead of clustering the data to signatures and moving the signatures, it perturbs each sample. In this manner it captures a finer notion of the perceptual similarity.   
  
The partition of the support to bins allows some relaxation of the TV notion. Therefore, instead of the TV, it may be interesting to consider the $L_1$ as a similarity distance on the measures after discretization.  
The example in Figure (\ref{fig:hists}) shows that this relaxation is quite rigid and that there is no single partition that captures the perceptual similarity. 
In general, the problem would remain even if bins with varying width were permitted. Namely, the problem is the choice of a single partition to measure similarity of a reference distribution to multiple distributions, while choosing multiple partitions would make the distances incomparable. Also note that defining a ``good'' partition is a difficult task, which is exasperated in higher dimensions.

The last group of statistics are scores established in machine learning: the $d_A$ distance presented by Kifer et al. that is based on the maximum discrepancy on a chosen subset of the support \cite{kifer2004detecting}, and Maximum Mean Discrepancy (MMD) by Gretton et al., which define discrepancy after embeddings the distributions to a Reproducing Kernel Hilbert Space (RKHS)\cite{gretton2007kernel}. These scores have corresponding statistical tests for the TSP; however, since their analysis is based on finite convergence bounds, in principle they may be modified to test similarity.   The $d_A$ captures some intuitive notion of similarity, however, to our knowledge, it is not known how to compute it for a general subset class \footnote{Most work with the $d_A$ has been with the subset of characteristic functions, and approximated by the error of a classifier.}.
The MMD captures the distance between the samples in some RKHS. While this distance perfectly defines an equality test,  it is not clear if it translates to a well defined similarity test. As an example, consider testing if the MMD is grater than some value larger than zero using the RBF kernel. To do so, the parameter $\sigma$ must be chosen in advance. Clearly, the result of the test is highly dependent on this choice, but it is not clear how it should be made. Contrarily, the PV's parameter $\epsilon$ is related to the data's input domain and may be chosen accordingly.     

\begin{figure}[t]
\centering
\subfigure[$\PV(\epsilon=0.1)=0$]{
\includegraphics[trim = 0cm 6cm 0cm 6cm, clip,scale=0.15]{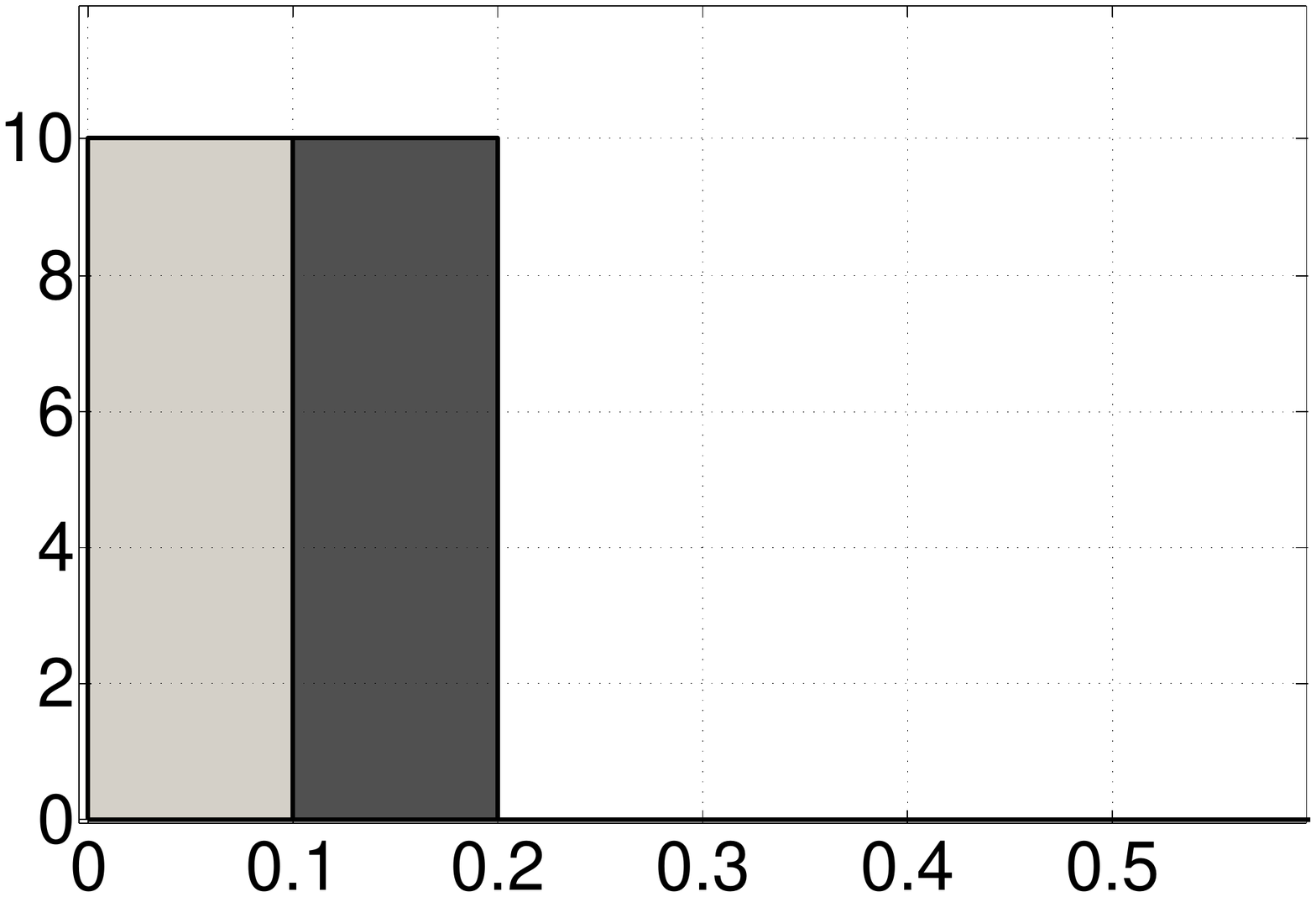}
\label{fig:hist1}
}
\subfigure[$\PV(\epsilon=0.1)=0$]{
\includegraphics[trim = 0cm 6cm 0cm 6cm, clip,scale=0.15]{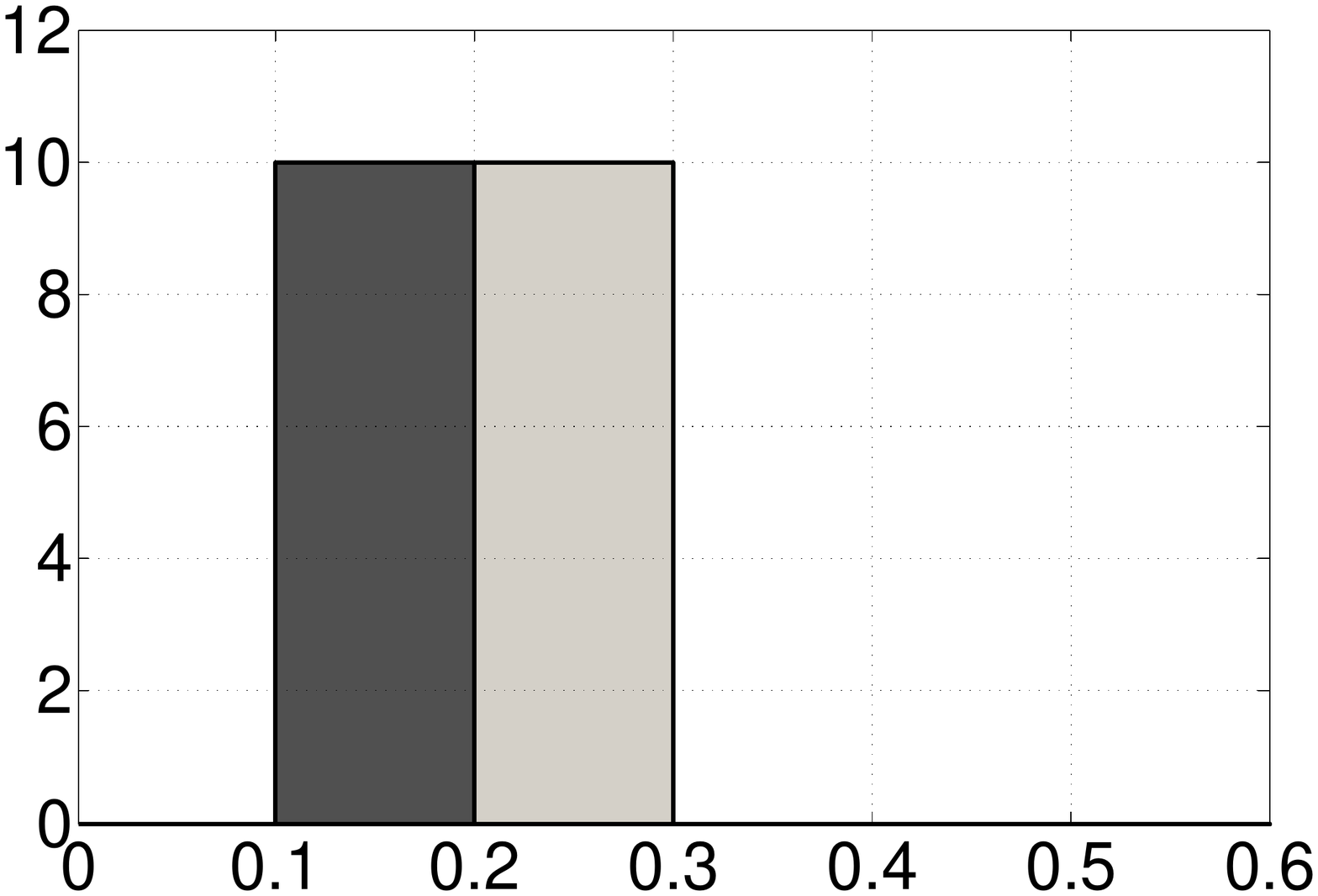}
\label{fig:hist2}
}
\subfigure[$\PV(\epsilon=0.1)=1$]{
\includegraphics[trim = 0cm 6cm 0cm 6cm, clip,scale=0.15]{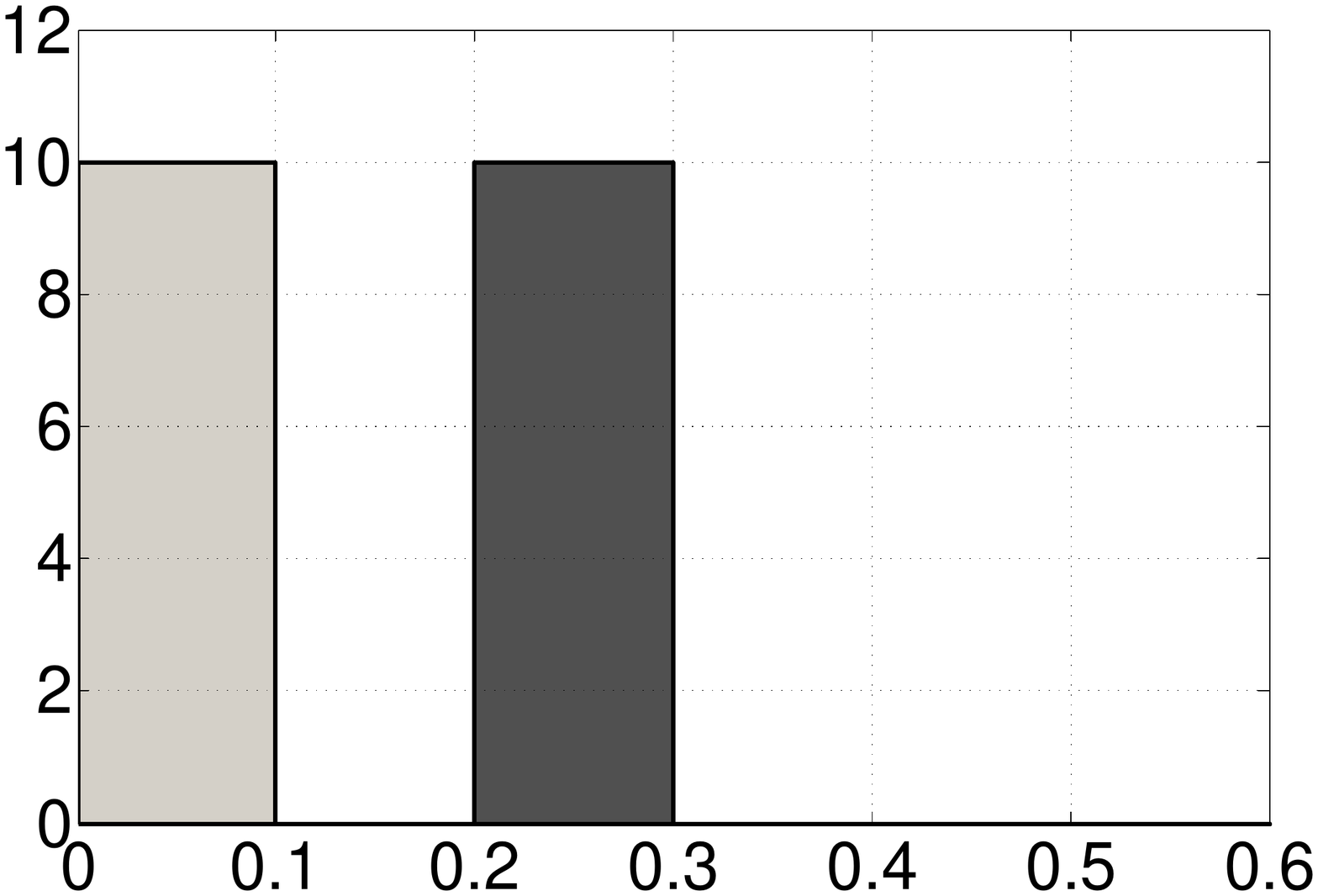}
\label{fig:hist3}
}
\vspace{-0.4cm}
\caption{ \small{Two distributions on $\mathbb{R}$: The PV captures the perceptual similarity of (a),(b) against the disimilarity in (c). The $L_1^1=1$ on $I_1=\{(0,0.1),(0.1,0.2),...\}$ for all cases; on $I_2=\{(0,0.2),(0.2,0.4),...\}$ it is $L_1^2(P_a,Q_a)= 0,L_1^2(P_b,Q_b)= 1,L_1^2(P_c,Q_c)=1$; and on $I_3=\{(0,0.3),(0.3,0.6),...\}$ it is $L_1^3(P_a,Q_a)= 0,L_1^3(P_b,Q_b)= 0,L_1^3(P_c,Q_c)=0$.}\label{fig:hists}}
\vspace{-0.4cm}
\end{figure}

\section{Analysis \label{Sec_Analysis} }
We present sample rate convergence analysis of the PV. The proofs of the theorems are provided in the supplementary material.  
When no clarity is lost, we omit ${\bf d}$ from the notation.
Our main theorem is stated as follows: 
\begin{theorem} \label{Thm_PVtailsInftyDist}
Suppose we are given two i.i.d. samples $S_1 =\{x_1,...,x_n\}\in \mathbb{R}^d$ and $S_2 =\{y_1,...,y_m\}\in \mathbb{R}^d$  generated by distributions $P$ and $Q$, respectively. Let the ground distance be ${\bf d}= \Vert\cdot\Vert_\infty$ and let $\mathcal{N}(\epsilon)$ be the cardinality of a disjoint cover of the distributions' support. Then,
for any $\delta\in(0,1)$, $N = \min(n,m)$, and $\eta= \sqrt{\frac{2(\log(2(2^{\mathcal{N}(\epsilon)}-2))+\log(1/\delta))}{N}}$ we have that
\begin{align*}
\mathbb{P} \left( \left|\PVh\brck{S_1,S_2,\epsilon}-  \PV\brck{P,Q,\epsilon}\right| \leq \ \eta \right) \geq 1-\delta.
\end{align*}
\end{theorem}
The theorem is defined using $\Vert\cdot\Vert_\infty$, but can be rewritten for other metrics (with a slight change of constants). The proof of the theorem exploits the form of the optimization Problem \ref{PVhat}.
We use the bound of Theorem \ref{Thm_PVtailsInftyDist} construct hypothesis tests. A weakness of this bound is its strong dependency on the dimension. Specifically, it is dependent on $\mathcal{N}(\epsilon)$, which for $ \Vert\cdot\Vert_\infty$ is  $O((1/\epsilon)^d)$: the number of disjoint boxes of volume $\epsilon^d$ that cover the support. Unfortunately, this convergence rate is inherent; namely, without making any further assumptions on the distribution, this rate is unavoidable and is an instance of the ``curse of dimensionality". 
In the following theorem, we present a lower bound on the convergence rate.
\begin{theorem}\label{Thm_Spheres}
 Let $P=Q$ be the uniform distribution on $\mathbb{S}^{d-1}$, a unit ($d-1$)--dimensional hyper-sphere.  Let $S_1 =\{x_1,...,x_N\}\sim P$ and $S_2 =\{y_1,...,y_N\}\sim Q$ be two i.i.d. samples. For any $\epsilon,\epsilon',\delta\in(0,1)$, $0\leq\eta<2/3$ and sample size
 $ \frac{\log(1/\delta)}{2(1-3\eta/2)^2}\leq N\leq \frac{\eta}{2} e^{d(1-\frac{\epsilon^2}{2})/2},$
 we have $PV\brck{P,Q,\epsilon'}=0$ and
\begin{align}
\mathbb{P}(\PVh\brck{S_1,S_2,\epsilon} > \eta ) \geq  1-\delta.
\end{align}
\end{theorem}
For example, for $\delta=0.01,\eta=0.5$, for any $37\leq N \leq 0.25 e^{d(1-\frac{\epsilon^2}{2})/2}$  we have that $\PVh>0.5$ with probability at least 0.99. The theorem shows that, for this choice of distributions, for a sample size that is smaller than $O(e^{d})$, there is a high probability that the value of $\PVh$ is far form PV.

It can be observed that the empirical estimate $\PVh$ is stable, that is, it is almost identical for two data sets differing on one sample. 
Due to its stability, applying McDiarmid inequality yields the following. 
\begin{theorem}\label{Thm_PVstability}
Let $S_1 =\{x_1,...,x_n\}\sim P $ and $S_2 =\{y_1,...,y_m\}\sim Q$ be two i.i.d. samples. Let $n\geq m$, then for any $\eta>0$
$$ \mathbb{P} \brck{ |\PVh\brck{S_1,S_2,\epsilon} -\mathbb{E}[\PVh\brck{n,m,\epsilon}]|\geq \eta }\leq e^{-\eta^2 m^2/4n},$$
where $\mathbb{E}[\PVh\brck{n,m,\epsilon}]$ is the expectation of $\PVh$ for a given sample size.
\end{theorem}
This theorem shows that the sample estimate of the PV converges to its expectation without dependence on the dimension. By combining this result with Theorem \ref{Thm_PVtailsInftyDist} it may be deduced that only the convergence of the bias -- 
the difference $|\mathbb{E}[\PVh(n,m,\epsilon)]-\PV(P,Q,\epsilon)|$ -- may be exponential in the dimension. This convergence is distribution dependent. However, intuitively, slow convergence is not always the case, for example when the support of the distributions lies in a lower dimensional manifold of the space.
To remedy this dependency we propose a bootstrapping bias correcting technique, presented in Section \ref{Sec_StatAnals}. A different possibility is to project the data to one dimension; due to space limitations, this extension of the PV is left out of the scope of this paper and presented in Appendix \ref{Sec_1Dproj} in the supplementary material.

\section{Statistical Inference \label{Sec_StatAnals}}

We construct two types of complementary procedures for hypothesis testing of similarity and dissimilarity\footnote{The two procedures are distinct, as, in general, lacking evidence to reject similarity is not sufficient to infer dissimilarity, and vice versa.}. In the first type of procedures, given $0\leq\theta<1$, we distinguish between the null hypothesis $\mathcal{H}^{(1)}_0:\PV(P,Q,\epsilon,{ \bf d})\leq \theta $, which implies similarity, and the alternative hypothesis  $\mathcal{H}^{(1)}_1:\PV(P,Q,\epsilon,{\bf d})>\theta$.
Notice that when $\theta=0$, this test is a relaxed version of the TSP.
Using $\PV(P,Q)=0$ instead of $P=Q$ as the null, allows for some distinction between the distributions, which gives the needed relaxation to capture similarity.
In the second type of procedures, we test whether two distributions are similar. To do so, we flip the role of the null and the alternative. Note that there isn't an equivalent of this form for the TSP, therefore we can not infer similarity using the TSP test, but only reject equality.
Our hypothesis tests are based on the finite sample analysis presented in Section \ref{Sec_Analysis}; see Appendix \ref{A1_HypTests} in the supplementary material for the procedures.

To provide further inference on the PV, we apply bootstrapping for approximations of Confidence Intervals (CI).
The idea of bootstrapping for estimating CIs is based on a two step procedure: approximation of the sampling distribution of the statistic by resampling with replacement from the initial sample -- the bootstrap stage -- following, a computation of the CI based on the resulting distribution. We propose to estimate the CI by Bootstrap Bias-Corrected accelerated (BCa) interval, which adjusts the simple percentile method to correct for bias and skewness \cite{efron1993introduction}. The BCa is known for its high accuracy; particularly, it can be shown, that the BCa interval converges to the theoretical CI with rate $O(N^{-1})$, where $N$ is the sample size.
Using the CI, a hypothesis test may be formed: the null $\mathcal{H}_0^{(1)}$ is rejected with significance $\alpha$ if the range $[0,\theta]\not\subset [\underline{CI},\overline{CI}]$. Also, for the second test, we apply the principle of CI inclusion \cite{EquivalenceTesting}, which states that if $[\underline{CI},\overline{CI}]\subset [0,\theta]$, dissimilarity is rejected and similarity deduced.
\section{Experiments \label{Sec_Exp}}

\subsection{Synthetic Simulations}
In our first experiment, we examine the effect of the choice of $\epsilon$ on the statistical power of the test. For this purpose, we apply significance testing for similarity on two univariate uniform distributions: $P\sim U[0,1]$ and $Q\sim U[\Delta(\epsilon),1+\Delta(\epsilon)]$, where $\Delta(\epsilon)$ is a varying size of perturbation. We considered values of $\epsilon=[0.1,0.2,0.3,0.4,0.5]$ and sample sizes up to $5000$ samples from each distribution. For each value $\epsilon'$, we test the null hypothesis $\mathcal{H}^{(1)}_0:PV(P,Q,\epsilon')=0$ for ten equally spaced values of $\Delta(\epsilon')$ in the range $[0,2\epsilon']$. In this manner, we test the ability of the PV to detect similarity for different sizes of perturbations. The percentage of times the null hypothesis was falsely rejected, i.e. the type-1 error, was kept at a significance level $\alpha=0.05$. The percentage of times the null hypothesis was correctly rejected, the power of the test, was estimated as a function of the sample 
size and averaged over 500 repetitions. We repeated the simulation using the tests based on the bounds as well as using BCa confidence intervals. 

The results in Figure (\ref{fig:syn_exp}) show the type-2 error of the bound based simulations. As expected, the power of the test increases as the sample size grows. Also, when finer perturbations need to be detected, more samples are needed to gain statistical power. For the BCa CI we obtained type-1 and type-2 errors smaller than 0.05 for all the sample sizes. This shows that the convergence of the estimated PV to its value is clearly faster than the bounds. Note that, given a sufficient sample size, any statistic for the TSP would have rejected similarity for any $\Delta>0$.  
 
\begin{figure}[t]
\centering
\subfigure[The Type-2 error for varying perturbation sizes and $\epsilon$ values. \label{fig:syn_exp}]{
\includegraphics[trim =0cm 6cm 0.4cm 7cm ,clip, width=0.31\textwidth] {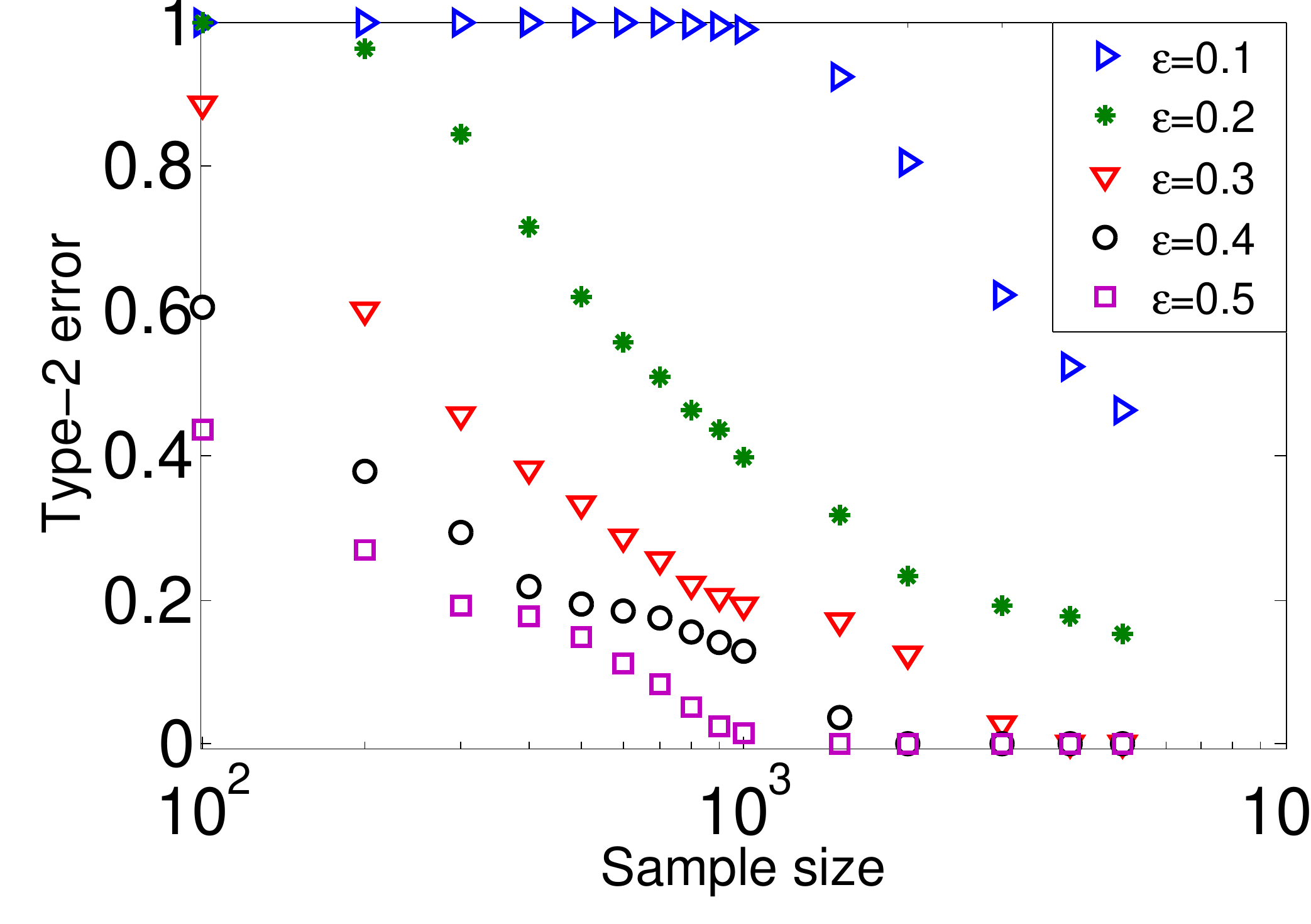}
}
\subfigure[Precision-Recall: Gait data.  \label{fig:Gait_PRC}]{
   \includegraphics[trim = 0cm 6cm 0.7cm 6cm, clip,width=0.31\textwidth] {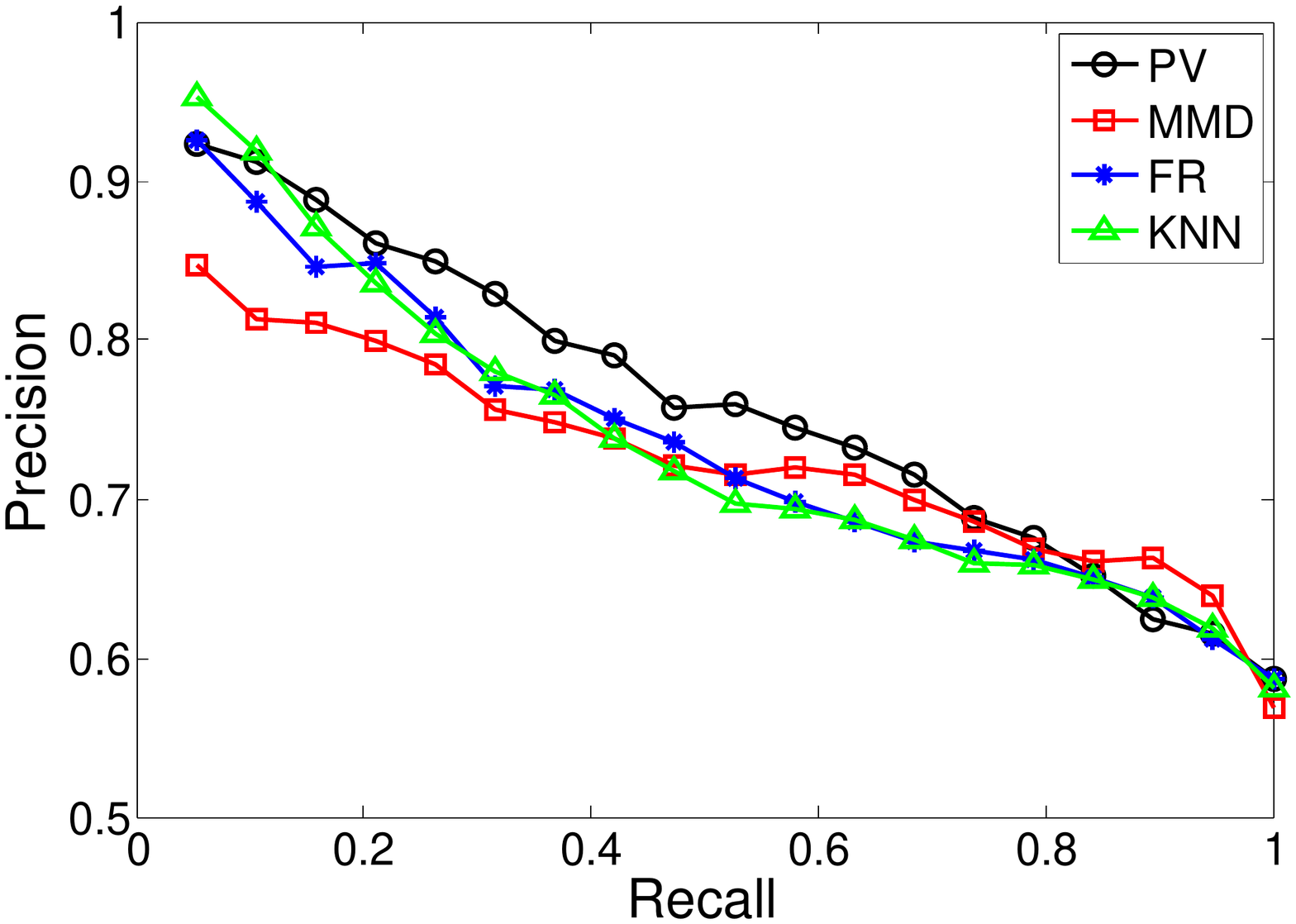}
}
 \subfigure[Precision-Recall: Video clips. \label{fig:Video_PRC}]{
  \includegraphics[trim = 0.8cm 6cm 0.7cm 6cm, clip,width=0.31\textwidth] {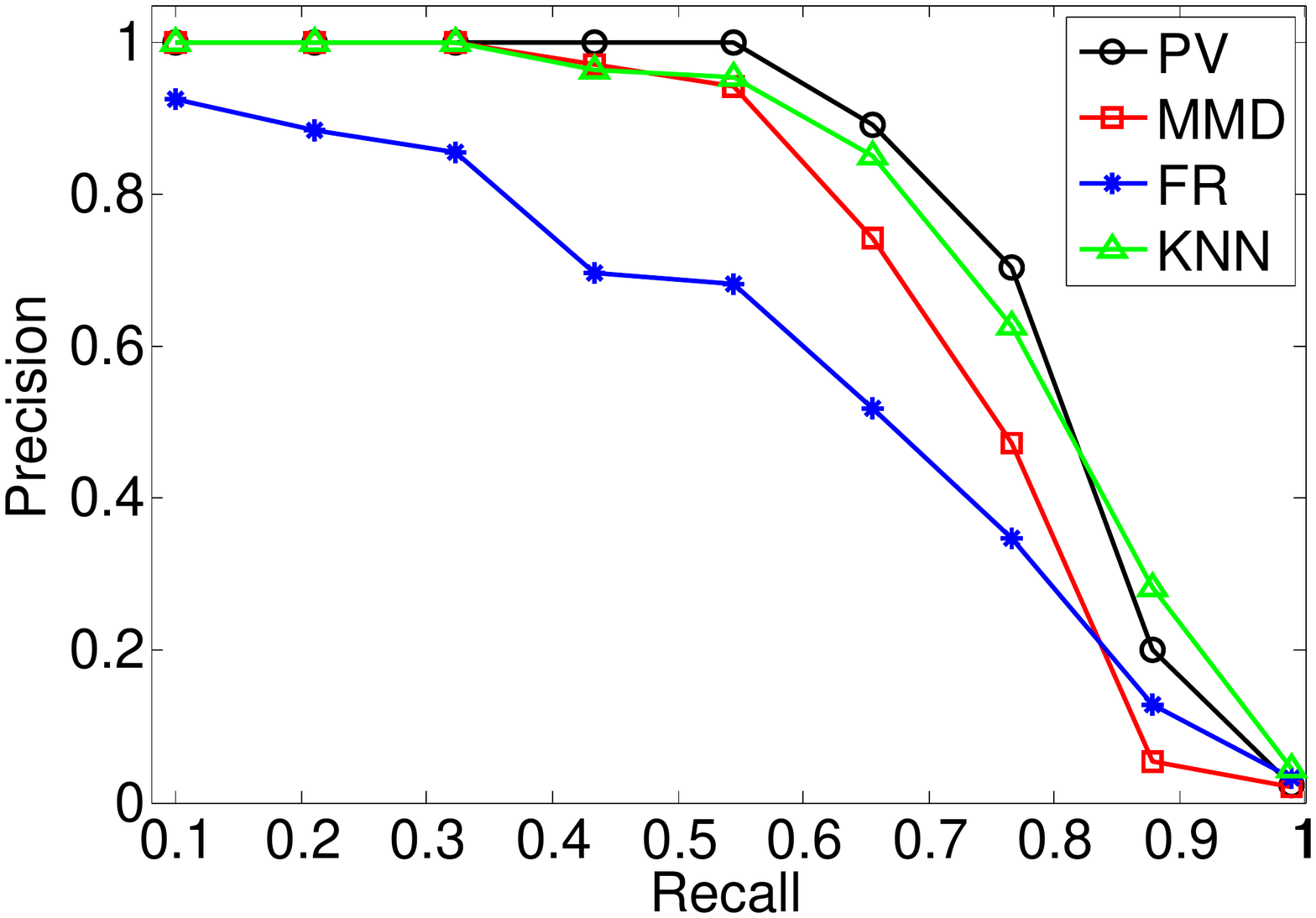}
}
\vspace{-0.3cm}
\end{figure}

\subsection{Comparing Distance Measures}
Next, we test the ability of the PV to measure similarity on real data.
To this end, we test the ranking performance of the PV score against other known distributional distances. We compare the PV to the multivariate extension of the Wald-Wolfowitz score of Friedman \& Rafsky ({\sc FR}) \cite{friedman1979multivariate} , Schilling's nearest neighbors score ({\sc KNN}) \cite{schilling1986multivariate}, and the Maximum Mean Discrepancy score of Gretton et al. \cite{gretton2007kernel} ({\sc MMD})\footnote{Note that the statistical tests of these measures test equality while the PV tests similarity and therefore our experiments are not of statistical power but of ranking similarity. Even in the case of the distances that may be transformed for similarity, like the MMD, there is no known function between the PV similarity to other forms of similarity. As a result, there is no basis on which to compare which similarity test has better performance.}.        
We rank similarity for the applications of video retrieval and gait recognition. 

The ranking performance of the methods was measured by precision-recall curves, and the Mean Average Precision (MAP). Let $r$ be the number of samples similar to a query sample. For each  $1\leq i\leq r$ of these observations, define $r_i\in [1,T-1]$ as its similarity rank, where $T$ is the total number of observations. The Average Precision is: $AP=1/r\sum_{i}i/r_i$, and the MAP is the average of the AP over the queries.
The tuning parameter for the methods -- $k$ for the {\sc KNN}, $\sigma$ for the {\sc MMD} (with RBF kernel), and $\epsilon$ for the PV -- were chosen by cross-validation. 
The Euclidian distance was used in all methods.

In our first experiment, we tested raking for video-clip retrieval. The data we used was collected and generated by \cite{shao2011exploring}, and includes 1,083 videos of commercials, each of about 1,500 frames (25 fps). Twenty unique videos were selected as query videos, each of which has one similar clip in the collection, to which 8 more similar clips were generated by different transformations: brightness increased/decreased, saturation increased/decreased, borders cropped, logo inserted, randomly dropped frames, and added noise frames. Lastly, each frame of a video was transformed to a 32-RGB representation.
We computed the similarity rate for each query video to all videos in the set, and ranked the position of each video. The results show that the PV and the KNN score are invariant to most of the transformations, and outperform the FR and MMD methods (Table \ref{MAPres} and Figure \ref{fig:Video_PRC}). We found that brightness changes were most problematic for the PV. For this type of distortion, the simple RGB representation is not sufficient to capture the similarity.

\begin{table}[t]\small
\begin{center}
\caption{ \small{MAP for Auslan, Video, and Gait data sets.
Average MAP ($\pm$ standard deviation)  computed on a random selection of $75\%$ of the queries, repeated 100 times.}}
\label{MAPres}
\begin{sc}
\begin{tabular}{lcccc}
\hline
Data set                     & $\PVh$             & KNN   & MMD & FR \\
\hline

Video                       &   {\bf 0.758 }$\pm 0.009$      & \bf{0.741} $\pm 0.014$   & $0.689\pm0.008$     & $0.563\pm0.019$ \\
\hline

Gait                        &   {\bf 0.792}$\pm0.021$      &  $0.736\pm0.014$    & $0.722\pm0.017$  &  $0.698\pm0.017$\\

Gait-F                      &   { \bf 0.844}$\pm0.017$    &  $0.750\pm0.015$    & $0.729 \pm 0.017$    &  $0.666\pm0.016$\\
Gait-M                      &   $0.679\pm0.024$          &  $0.712\pm0.017$    & $0.716\pm0.031$  &  {\bf 0.799} $\pm0.016$ \\

\hline
\end{tabular}
\end{sc}
\end{center}
\vspace{-2ex}
\end{table}

We also tested gait similarity of female and male subjects; same gender samples are assumed similar. We used gait data that was recorded by a mobile phone, available at \cite{frank2010gaitdata}. The data consists of two sets of 15min walks of 20 individuals, 10 women and 10 men.  As features we used the magnitude of the triaxial accelerometer.We cut the raw data to intervals of approximately 0.5secs, without identification of gait cycles. In this manner, each walk is represented by a collection of about 1500 intervals. An initial scaling to [0,1] was performed once for the whole set. The comparison was done by ranking by gender the 39 samples with respect to a reference walk.

The precision-recall curves in Figure \ref{fig:Gait_PRC} show that the PV is able to retrieve with higher precision in the mid-recall range. For the early recall points the PV did not show optimal performance; Interestingly, we found that with a smaller $\epsilon$, the PV had better performance on early recall points. This behavior reflects the flexibility of the PV: smaller $\epsilon$ should be chosen when the goal is to find very similar instances, and larger when the goal is to find higher level similarity. The MAP results presented in Table \ref{MAPres} show that the PV had better performance on the female subjects. From examination of the subject information sheet we found that the range of weight and hight within the female group is 50-77Kg and 1.6-1.8m, while within the male group it is 47-100Kg and 1.65-1.93m; that is, there is much more variability in the male group. This information provides a reasonable explanation to the PV results, as it appears that a subject from the male group may have a gait 
that is as dissimilar to the gait of a female subject as it is to a different male. In the female group the subjects are more similar and therefore the precision is higher.


\section{Discussion}
We proposed a new score that measures the similarity between two multivariate distributions, and assigns to it a value in the range [0,1]. The sensitivity of the score, reflected by the parameter $\epsilon$, allows for flexibility that is essential for quantifying the notion of similarity. The PV is efficiently estimated from samples. Its low computational complexity relies on its simple binary classification of points as neighbors or non-neighbor points, such that optimization of distances of faraway points is not needed. In this manner, the PV captures only the essential information to describe similarity.
Although it is not a metric, our experiments show that it captures the distance between similar distributions as well as well known distributional distances. 
Our work also includes convergence analysis of the PV. Based on this analysis we provide hypothesis tests that give statistical significance to the resulting score. While our bounds are dependent on the dimension, when the intrinsic dimension of the data is smaller than the domains dimension, statistical power can be gained by bootstrapping. In addition, the PV has an intuitive interpretation that makes it an attractive score for a meaningful statistical testing of similarity.
Lastly, an added value of the PV is that its computation also gives insight to the areas of discrepancy; namely, the areas of the unmatched samples. In future work we plan to further explore this information, which may be valuable on its own merits.


\bibliographystyle{unsrt}
\bibliography{PVbibNIPS2012}

\begin{thebibliography}{10}

\bibitem{Monge1781}
G.~Monge.
\newblock M\'emoire sur la th{\'e}orie des d{\'e}blais et de remblais.
\newblock Histoire de l'Academie Royale des Sciences de Paris, avec les
  Memoires de Mathematique et de Physique pour la meme annee, 1781.

\bibitem{ruschendorf2007monge}
L.~R{\"u}schendorf.
\newblock Monge--kantorovich transportation problem and optimal couplings.
\newblock {\em Jahresbericht der DMV}, 3:113--137, 2007.

\bibitem{schrijver1998theory}
A.~Schrijver.
\newblock {\em Theory of linear and integer programming}.
\newblock John Wiley \& Sons Inc, 1998.

\bibitem{rubner1998metric}
Y.~Rubner, C.~Tomasi, and L.J. Guibas.
\newblock A metric for distributions with applications to image databases.
\newblock In {\em Computer Vision, 1998. Sixth International Conference on},
  pages 59--66. IEEE, 1998.

\bibitem{Ahuja93}
R.K. Ahuja, L.~Magnanti, and J.B. Orlin.
\newblock {\em Network Flows: Theory, Algorithms, and Applications},
  chapter~12, pages 469--473.
\newblock Prentice Hall, 1993.

\bibitem{friedman1979multivariate}
J.H. Friedman and L.C. Rafsky.
\newblock Multivariate generalizations of the {W}ald-{W}olfowitz and {S}mirnov
  two-sample tests.
\newblock {\em Annals of Statistics}, 7:697--717, 1979.

\bibitem{schilling1986multivariate}
M.F. Schilling.
\newblock Multivariate two-sample tests based on nearest neighbors.
\newblock {\em Journal of the American Statistical Association}, pages
  799--806, 1986.

\bibitem{kifer2004detecting}
D.~Kifer, S.~Ben-David, and J.~Gehrke.
\newblock Detecting change in data streams.
\newblock In {\em Proceedings of the Thirtieth international conference on Very
  large data bases}, pages 180--191. VLDB Endowment, 2004.

\bibitem{gretton2007kernel}
A.~Gretton, K.~Borgwardt, B.~Sch{\"o}lkopf, M.~Rasch, and E.~Smola.
\newblock A kernel method for the two sample problem.
\newblock In {\em Advances in Neural Information Processing Systems 19}, 2007.

\bibitem{efron1993introduction}
B.~Efron and R.~Tibshirani.
\newblock {\em An introduction to the bootstrap}, chapter~14, pages 178--188.
\newblock Chapman \& Hall/CRC, 1993.

\bibitem{EquivalenceTesting}
S.~Wellek.
\newblock {\em Testing Statistical Hypotheses of Equivalence and
  Noninferiority; 2nd edition}.
\newblock Chapman and Hall/CRC, 2010.

\bibitem{shao2011exploring}
J.~Shao, Z.~Huang, H.~Shen, J.~Shen, and X.~Zhou.
\newblock Distribution-based similarity measures for multi-dimensional point
  set retrieval applications.
\newblock In {\em Proceeding of the 16th ACM international conference on
  Multimedia MM 08}, 2008.

\bibitem{frank2010gaitdata}
J.~Frank, S.~Mannor, and D.~Precup.
\newblock Data sets: Mobile phone gait recognition data, 2010.

\bibitem{weissman2003inequalities}
T.~Weissman, E.~Ordentlich, G.~Seroussi, S.~Verdu, and M.J. Weinberger.
\newblock Inequalities for the l1 deviation of the empirical distribution.
\newblock {\em Hewlett-Packard Labs, Tech. Rep}, 2003.

\bibitem{Boyd2004}
S.~Boyd and L.~Vandenberghe.
\newblock {\em Convex Optimization}, chapter~5, pages 258--261.
\newblock Cambridge University Press, New York, NY, USA, 2004.

\end{thebibliography}

\newpage

\appendix

\section{Supplementary Material}

\subsection{Hypothesis Testing Procedures \label{A1_HypTests}}
The statistical tests in this section are based on the convergence bounds in Section \ref{Sec_Analysis}.
\paragraph{Notations} Throughout this section the probabilities $\mathbb{P}_0$ and $\mathbb{P}_1$ represent the probability
conditioned on the null hypothesis $\mathcal{H}_0$, and the alternative hypothesis $\mathcal{H}_1$.

The following procedure tests the hypothesis $\mathcal{H}^{(1)}_0:\PV(P,Q,\epsilon)\leq \theta$ against the alternative $\mathcal{H}^{(1)}_1:\PV(P,Q,\epsilon)>\theta.$

\fbox{
\smallskip
\begin{minipage}{1\textwidth}
\begin{procedure}\label{procedure1}
Similarity Testing Based on $\PVh$.\\
{\bf Input}: $\epsilon$, $\theta$ and significance level $\alpha$.
\begin{enumerate}
\item Sample $S_1 = \{x_1,...,x_n\}\sim P$ and $S_2=\{y_1,...,y_m\}\sim Q$ (define $N=\min(n,m)$).
\item Normalize the data to be in $[0,1]^d$.
\item Compute $\widehat{PV}(S_1,S_2,\epsilon,\Vert \cdot \Vert_\infty)$ by Algorithm \ref{Alg_PVhat}.

\item Compute $t = \sqrt{\frac{(2\log(2(2^{(1/\epsilon)^d}-2))+2\log(1/\alpha)}{N}}$.
\end{enumerate}
{\bf Output}: Reject $\mathcal{H}^{(1)}_0$ if $$\widehat{PV}(S_1,S_2,\epsilon,\Vert \cdot \Vert_\infty) > t+\theta$$.
\end{procedure}
\end{minipage}
}

The probability to reject $\mathcal{H}^{(1)}_0$ by applying Procedure 1 when in fact it holds -- also known as the Type 1 error -- is bounded in the following corollary. 
\begin{corollary}\label{lem_type1Proc1}
Assume that for a given $\epsilon$ and $\theta$ values $\mathcal{H}^{(1)}_0:\,PV(P,Q,\epsilon,{ \bf d})\leq \theta$ holds.
Then for the threshold $t$ of Procedure \ref{procedure1} and any $\alpha\in(0,1)$
we have that
\begin{align}\label{Eq_type1Proc1}
\mathbb{P}_0 \left(\widehat{PV}(S_1,S_2,\epsilon,\Vert\cdot\Vert_\infty) \geq  t +\theta \right) \leq \alpha.
\end{align}
Moreover, the procedure is consistent: when $n,m \rightarrow \infty$ we have that $t\rightarrow 0$ and $\mathbb{P}_1(\widehat{PV}(S_1,S_2,\epsilon,\Vert\cdot\Vert_\infty) > \theta )=1$.
\end{corollary}
The corollary is a direct result of Theorem \ref{Thm_PVtailsInftyDist}.

Next, we consider the probability that Procedure \ref{procedure1} fails to reject $\mathcal{H}^{(1)}_0$  when the alternative hypothesis $\mathcal{H}^{(1)}_1$ holds, also known as the Type 2 error. Unfortunately, it is not possible to bound this probability for a finite sample of \emph{any} two distributions. To see this, consider the following example: let $P,Q$ be two distributions with $PV(P,Q,\epsilon)>0$, but differ only in an area of very low probability. Then, for any finite sample size, there is a high probability that the samples are identical, resulting in $\widehat{PV}(S_1,S_2,\epsilon)=0$. As a result, the null hypothesis will not be rejected even though $\mathcal{H}^{(1)}_1$ holds.

However, if the PV is larger than some constant the Type 2 error is bounded.
\begin{corollary}\label{lem_type2Proc1}
For $PV(P,Q,\epsilon,{\bf d})>\theta+t+b$, with $t$ of Procedure \ref{procedure1}, and $b = \sqrt{\frac{2(\log(2(2^{(1/\epsilon)^d}-2))+2\log(1/\beta)}{N}}$ we have that
$$ \mathbb{P} \left( \widehat{PV}(S_1,S_2,\epsilon,\Vert\cdot\Vert_\infty) > t+\theta \right)\geq 1-\beta.$$
\end{corollary}
Note that as $N$ grows, the values of $b$ and $t$ get smaller, and the lower bound $PV(P,Q,\epsilon,{\bf d})>\theta+t+b$ decreases.
\begin{proof}
We have that
\begin{align*}
&\mathbb{P}_1\left( \widehat{PV}(S_1,S_2,\epsilon,\Vert\cdot\Vert_\infty) > t+\theta \right)=\mathbb{P}_1\left( \widehat{PV}(S_1,S_2,\epsilon,\Vert\cdot\Vert_\infty) > b+t+\theta-b \right) \geq \\
&\mathbb{P}\left( \widehat{PV}(S_1,S_2,\epsilon,\Vert\cdot\Vert_\infty) > PV(P,Q,\epsilon,\Vert\cdot\Vert_\infty) -b \right) \geq 1-\beta.
\end{align*}
The first inequality holds by inserting the assumption on $PV$, and the second holds by applying the convergence bound of Theorem \ref{Thm_PVtailsInftyDist}.
\end{proof}

To give an estimate of the sample size needed for the procedure, first define the effect size $\theta_0$: the minimal value of PV that is significant.
Given $\theta_0$, set the sample size so that
$$N \geq  \frac{4\log(2(2^{(1/\epsilon)^d}-2))+2\log(1/\alpha)+2\log(1/\beta)}{\theta_0^2}.$$
Using this size ensures a false positive rate bounded by $\alpha$ (Corollary \ref{lem_type1Proc1}), and a false negative rate bounded by $\beta$ (Corollary \ref{lem_type2Proc1}).

The second test we consider is an equivalence type test \cite{EquivalenceTesting}. Equivalence is achieved when $\PV(P,Q,\epsilon)< \theta$, for some chosen $\theta$, and may be obtained by switching the roles of the null and the alternative of Procedure \ref{procedure1}. Namely, to claim similarity we need to reject $\mathcal{H}^{(2)}_0: \PV(P,Q,\epsilon)\geq \theta$. To test this hypothesis, a similar procedure to Procedure 1 may be applied, with a principal difference in the rejection area, which is changed to $\widehat{PV}(S_1,S_2,\epsilon,\Vert \cdot \Vert_\infty) < \theta -t$.

\subsection{1D Projections \label{Sec_1Dproj}}
We present a method to gain insight on the value of the PV by multiple random projections to one dimension.
The PV between two distributions is not retained after projection to a single dimension, as the projection contracts the distance between the points. However, we show that multiple projections can still aid to distinguish between two situations: $\PV(P,Q,\epsilon)=0$ and $\PV(P,Q,\epsilon)\neq0$ \footnote{Recall that PV=0 not only when the distributions are equal, but also when they are $\epsilon$ similar.}.
First, we define a score that is based on the value of the PV after projections.
\begin{definition} \label{def_PPV}
Let $f_i:\mathbb{R}^d\rightarrow \mathbb{R}$ for $i=1,...,K$, define random projection mappings.
Let $X$ and $Y$ be random variables with distributions $P$ and $Q$. The maximum projected score of two distributions $P$ and $Q$ is
 $$PPV_K(P,Q,\epsilon)=\max_{i=1,...,K}\PV(f_i(X),f_i(Y),\epsilon).$$
For two samples $S_1\sim P$ and $S_2\sim Q$ the score is
 $$\PPV_K(S_1,S_2,\epsilon)=\max_{i=1,...,K}\PVh(f_i(S_1),f_i(S_2),\epsilon).$$
\end{definition}
The next theorem presents the convergence of $\PPV_K$ to zero for distributions with $PV(P,Q)=0$.
\begin{theorem} \label{Thm_PPV1type1}
    Let $P$ and $Q$ be two distributions on the space $([0,1]^d,\bf{d})$, and $S_1=\{x_1,...,x_n\}\sim P$ and $S_2=\{y_1,...,y_m\}\sim Q$ two i.i.d. samples ($N=\min(n,m)$). Perform $K$ i.i.d. random projections of samples $S_1$ and $S_2$ to one dimension.  If  $PV(P,Q,\epsilon,{\bf d})=0$, then for any $\delta\in(0,1)$, with probability at least $1-\delta$
    $$ \PPV_K(S_1,S_2,\epsilon)  \leq \sqrt{\frac{2\log(2K(2^{1/\epsilon}-2)/\delta)}{N}}.$$
\end{theorem}

\begin{proof}
Given $PV(P,Q,\epsilon,{\bf d})=0$, we have that for all $K$ projections  $PPV_{1i}(P,Q,\epsilon,|\cdot|)=PV(P,Q,\epsilon,{\bf d})=0$, as the projection to 1D is a non-expansion.

In the following we denote by $\mathbb{P}_0(A)$ the probability of event $A$ under the assumption $PV(P,Q,\epsilon,{\bf d})=0$.
Denote $\PPV_{i}(\epsilon)= \PVh(f_i(S_1),f_i(S_2),\epsilon)$ as the value of PV obtained due to the $i$th projection.

We bound the probability of the event $\PPV_K(S_1,S_2,\epsilon)\geq \eta$:
\begin{align}\label{eq:PPV_eq1}
           &\mathbb{P}_0 \left( \max_{1 \leq i \leq K}  \widehat{PPV}_{i}(\epsilon) \geq \eta  \right) = 
           \mathbb{P}_0 \left(\exists 1 \leq i \leq K : \widehat{PPV}_{i}(\epsilon) \geq \eta \right)
           \leq \sum_{i=1}^{K} \mathbb{P}_0 \left(  \widehat{PPV}_{i}(\epsilon) \geq \eta \right), 
\end{align}
where the last inequality is obtained by applying the union bound.

Combining (\ref{eq:PPV_eq1}) with Theorem \ref{Thm_PVtailsInftyDist}, we have that for any $\eta\in(0,1)$
\begin{align*}
& \mathbb{P}_0 \left( \max_{1 \leq i \leq K}  \widehat{PPV}_{i}(\epsilon) \geq \eta  \right) 
\leq K\max_{1 \leq i \leq K} \mathbb{P}_0 \left(  \widehat{PPV}_{i}(\epsilon) \geq \eta \right)
\leq  2K(2^{1/\epsilon}-2)e^{-N\eta^2/2}.
\end{align*}
Setting $\delta =2K(2^{1/\epsilon}-2)e^{-N\eta^2/2}$ concludes the proof.
\end{proof}

For $\PV(P,Q)>0$ , we provide a similar lower bound on the maximum score. We will need a further assumption for this bound.
\begin{assumption}\label{Asmp_supPPV}
Given distributions $P$ and $Q$ with $\PV(P,Q,\epsilon)>0$, they are \textbf{1D distinguishable} if $\lim_{K\rightarrow\infty} PPV_{K}(P,Q,\epsilon)>0$ almost surely.
\end{assumption}
This assumption ensures that the difference in the PV value exists in at least one projection.
 \begin{theorem}\label{Thm_PPV1type2}
Let $P$ and $Q$ be two distributions on the space $([0,1]^d,\bf{d})$. Given $i=1,...,K$ i.i.d. samples $S_{i1}=\{x_{i1},...,x_{in}\}\sim P$ and $S_{i2}=\{y_{i1},...,y_{im}\}\sim Q$, and $K$ mappings $f_i$, for any two distribution that fulfill Assumption \ref{Asmp_supPPV} there exists
 some $q\in(0,1)$, for which for any $\delta\in(0,1)$ with probability at least $1-\left(  q-q\delta+\delta \right)^K$
$$\PPV(\{S_{i1}\},\{S_{i2}\},\epsilon)  \geq \sqrt{\frac{2\log(2K(2^{1/\epsilon}-2)/\delta)}{N}}.$$
\end{theorem}
The notation $\PPV(\{S_{i1}\},\{S_{i2}\},\epsilon)$ denotes the maximum taken over the projections of the $K$ sets.
Notice that $q-q\delta+\delta<1$, and therefore is an exponential decay in the number of projections $K$.

\begin{proof}

Let $f:\mathbb{R}^d\rightarrow \mathbb{R}$ define a random projection mappings.
Let $X$ and $Y$ be random variables generated by $P$ and $Q$.
Denote $\PV_{1}= \PV(f(X),f(Y),\epsilon),$ and
$\PVh_{1}(\epsilon)= \PVh(f(S_1),f(S_2),\epsilon)$.
Note that there are two sources of randomization, the sample's and the projection's, and therefore $\PV_{1}$ is also a random variable.

We have that
\begin{align} \label{eq1_PPV1}
&\mathbb{P}(\PPV(\{S_{i1}\},\{S_{i2}\},\epsilon) \leq \eta ) = 
\mathbb{P}( \max_{1\leq i \leq K} \PVh(f_i(S_{i1}),f_i(S_{i2}),\epsilon) \leq \eta) =\\
&\mathbb{P}( \forall \, 1\leq i \leq K \,, \PVh(f_i(S_{i1}),f_i(S_{i2}),\epsilon)\leq \eta) \nonumber=
 [ \mathbb{P}(\PVh_{1}(\epsilon)\leq \eta )]^K ,\nonumber
\end{align}
where the last equality holds due to the independence of the events.
Next, we bound the probability $\mathbb{P}( \PVh_{1}(\epsilon)\leq \eta )$.
We define complementary events $A: \PV_{1}(\epsilon) \geq 2\eta$ and $A^c :\PV_{1}(\epsilon)< 2\eta$.
\begin{align}\label{eq2_PPV1}
&  \mathbb{P}( \PVh_{1}(\epsilon) \leq \eta )   = 
P(A) P( \PVh_{1}(\epsilon) \leq \eta \,|\, A )+ P(A^c) P(\PVh_{1}(\epsilon) \leq \eta \,|\,A^c ) \\
& \leq P(A) P\left( \PVh_{1}(\epsilon) \leq \PV_{1} - \eta \,|\, A   \right)  
+ P(A^c)  P\left( \PVh_{1}(\epsilon) \leq \eta \,|\,A^c \right)  \nonumber\\
&\leq P(A) 2(2^{1/\epsilon}-2)e^{-N\eta^2/2} +
 P(A^c)  
\leq P(A) 2K(2^{1/\epsilon}-2)e^{-N\eta^2/2} +1-P(A).\nonumber
\end{align}

The inequality before last is obtained by applying Theorem \ref{Thm_PVtailsInftyDist} for any $\eta\in(0,1)$.
For $\delta = 2K(2^{1/\epsilon}-2)e^{-N\tilde{\eta}^2/2}$ we have  $\tilde{\eta}= \sqrt{\frac{2\log(2K(2^{1/\epsilon}-2)\delta)}{N} }.$
Substituting this $\tilde{\eta}$ to (\ref{eq2_PPV1}) results in
\begin{align}\label{eq3_PPV1}
\mathbb{P}( \PV_{1}(\epsilon) \leq \tilde{\eta} )\leq 1 - (1-\delta)P(A).
\end{align}

Let $p(\eta,\epsilon) = \mathbb{P}(\PV_1(P,Q,\epsilon) \leq \eta) $ be the distribution of the projected PV. Clearly, $p(\eta,\epsilon)$ depends on the generating distributions $P$ and $Q$, and its support is
$[0,\sup_i (PPV_{i}(P,Q,\epsilon))]$.
We assume that $\sup_i (PPV_{i}(P,Q,\epsilon))>0$, and therefore there must be some $q\in(0,1)$ for which
\begin{align}\label{eq4_PPV1}
  \mathbb{P}(\PV_1(P,Q,\epsilon) < 2\tilde{\eta}) \leq q.
\end{align}

Combining the results of (\ref{eq1_PPV1}), (\ref{eq3_PPV1}) and (\ref{eq4_PPV1}), we have that for any $0<\delta<1$
\begin{align*}
&\mathbb{P}(\PPV(\{S_{i1}\},\{S_{i2}\},\epsilon) \leq \tilde{\eta} ) 
\leq \left(1 - (1-\delta)\mathbb{P}(\PV_1(\epsilon) \geq 2\tilde{\eta}) \right)^K 
\leq ( q-q\delta+\delta)^K.
\end{align*}

Therefore, with probability at least $1-( q-q\delta+\delta)^K$
\begin{align*}
&\PPV(\{S_{i1}\},\{S_{i2}\},\epsilon) \geq \sqrt{\frac{2\log(2K(2^{1/\epsilon}-2)/\delta)}{N}} .
\end{align*}
Note that for any $q<1$ results in $q-q\alpha+\alpha<1$, and therefore we get exponential decay.
For example for $q=1/2$ ($2\tilde{\eta}$  is smaller then the median) we have $\left(q-q\delta+\delta\right)^K = \left(\frac{1+\delta}{2} \right)^K$.
\end{proof}

Theorems \ref{Thm_PPV1type1} and \ref{Thm_PPV1type2} are complementary, and may be used together to infer whether or not $\PV(P,Q)=0$. Next, we describe the suitable hypothesis testing procedure for this goal.
Procedure \ref{procedure2} provides statistical tests based on the score $\PPV_K$ (Definition \ref{def_PPV}).
The procedure tests an hypothesis of the first type with $\theta=0$: $\mathcal{H}^{(1)}_0:PV(P,Q,\epsilon)=0$ against $\mathcal{H}^{(1)}_1:PV(P,Q,\epsilon)>0.$

\fbox{
\smallskip
\begin{minipage}{1\textwidth}
\begin{procedure}\label{procedure2}
 Similarity testing based on $\PPV_K$.\\
{\bf Input}: $\epsilon$ level, number of projections $K$, and significance level $\alpha$.\\
{\bf For} $i=1,...,K$ {\bf do}
\begin{enumerate}
\item    Sample $S_{i1} = \{x_1,...,x_n\}\sim P$ and $S_{i2}=\{y_1,...,y_m\}\sim Q$ i.i.d. examples on $[0,1]^d$.
\item    Sample a unit random vector $r_i\in\mathbb{S}^{d-1}$.
\item    Project to $1D$: $Ps_{i1} = \{r_i^Tx_1,...,r_i^Tx_n\}$ and $Ps_{i2}=\{r_i^Ty_1,...,r_i^Ty_m\}.$
\item    Compute $\PVh(Ps_{i1},Ps_{i2},\epsilon)$.
\end{enumerate}
{\bf end for}\\
 Compute $\PPV_K=\max_{i=1,...,K}\PVh(Ps_{i1},Ps_{i2},\epsilon)$.\\
 Compute
$t = \sqrt{\frac{\log(K) + 2\log(2(2^{1/\epsilon}-2))+2\log(1/\alpha)}{N}}$, for $N=\min(n,m)$. \\
{\bf Output}: Reject $\mathcal{H}_0$ if $\PPV_K > t$.\\
\end{procedure}
\end{minipage}
}

The next corollary bounds the Type 1 error of Procedure \ref{procedure2}, and shows that the test is consistent.
\begin{corollary}
Assume that the null hypothesis holds: $\mathcal{H}_0:\,PV(P,Q,\epsilon,{ \bf d})=0$.
Then, for the threshold $t$ of Procedure \ref{procedure2} and any $\alpha\in(0,1)$
we have that
\begin{align}\label{Eq_type1Proc2}
\mathbb{P}_0 \left(\PPV_K(S_1,S_2,\epsilon) \geq  t \right) \leq \alpha.
\end{align}
Moreover, when $N\rightarrow \infty, K\rightarrow \infty$ and $\frac{\log(K)}{N}\rightarrow 0$ we have that
$\mathbb{P}_1(\PPV_K(S_1,S_2,\epsilon) > t )=1$.
\end{corollary}
The bound is obtained by Theorem \ref{Thm_PPV1type1}.
The consistency is conditioned on Assumption \ref{Asmp_supPPV}, and obtained by Theorem \ref{Thm_PPV1type2}.

Theorem \ref{Thm_PPV1type2} bounds the Type 2 error of Procedure \ref{procedure2}, which is dependent on the number of projections $K$, and the fraction $q$ that is distribution dependent. The bound exponentially decays as $K$ grows, and therefore, to gain statistical power, a larger number of projections can be used.

\newpage
\subsection{Proof of Theorem \ref{Thm_PVtailsInftyDist}}
We restate the theorem for clarity:

\begin{reptheorem}{Thm_PVtailsInftyDist}
Suppose we are given two i.i.d. samples $S_1 =\{x_1,...,x_n\}\in \mathbb{R}^d$ and $S_2 =\{y_1,...,y_m\}\in \mathbb{R}^d$  generated by distributions $P$ and $Q$, respectively. Let the ground distance be ${\bf d}= \Vert\cdot\Vert_\infty$ and let $\mathcal{N}(\epsilon)$ be the cardinality of a disjoint cover of the distributions' support. Then,
for any $\delta\in(0,1)$, $N = \min(n,m)$, and $\eta = \sqrt{\frac{2(\log(2(2^{\mathcal{N}(\epsilon)}-2))+\log(1/\delta))}{N}}$ we have that
\begin{align*}
\mathbb{P} \left( \left|\PVh\brck{S_1,S_2,\epsilon}-  \PV\brck{P,Q,\epsilon}\right| \leq \ \eta \right) \geq 1-\delta.
\end{align*}
\end{reptheorem}

Before proving the theorem we present the required definitions and lemmas. The proofs of the lammas are presented immediately after the proof of the theorem. We assume the domain is totally bounded, and, for simplicity of presentation, we assume the metric space is $([0,1]^d , \bf{d}_\infty=\Vert\cdot\Vert_\infty)$.

We define a discretization on the support of the distributions.
\begin{definition}[Discretization]
The \textbf{$\epsilon$-discretization } over the space $([0,1]^d , \bf{d}_\infty=\Vert\cdot\Vert_\infty)$ is a partition on the set $C(\epsilon)=\{a_1,...,a_{N} \}$, with cardinality $N = (1/\epsilon)^d.$ Each element in $C(\epsilon)$ is the center of a box of volume $\epsilon^d$. The boxes do not intersect, and their union covers $[0,1]^d$. Each $a_i\in C(\epsilon)$ has a density equal to the distribution's mass in its neighborhood: $B(a_i,{ \bf d}_\infty,\epsilon)=\{z : \, {\bf d}_\infty(a_i,z)\leq \epsilon/2\}$.
\end{definition}

We refer to the resulting discretized versions of the distributions $P$
and $Q$ as $\mu_1(\epsilon)$, $\mu_2(\epsilon)$ respectively. Also, let $\hat{\mu}_1(\epsilon)$, $\hat{\mu}_2(\epsilon)$ be the \emph{histograms} of the samples $S_1$ and $S_2$, defined on the $\epsilon$-discretization $C(\epsilon)$.

The proof of Theorem \ref{Thm_PVtailsInftyDist} is based on formulating the relations between  $\PVh(S_1,S_2)$ and $\PV(\hat{\mu}_1,\hat{\mu}_1)$, and between $PV(P,Q)$ and $PV(\mu_1,\mu_2)$; then, turning to the discrete versions, bounding the difference between $\PV(\hat{\mu}_1,\hat{\mu}_1)$ and $PV(\mu_1,\mu_2)$.

The relation between the different versions of the PV, continuous, discrete and sampled, is provided in the next lemma.
\begin{lemma}
\label{PVineq}
Let $S_1 =\{x_1,...,x_n\}\sim~P$ and $S_2 =\{y_1,...,y_m\}\sim~Q$ be two samples. Let $\mu_1(\nu)$ and $\mu_2(\nu)$ be the $\nu$-discretizations of $P$ and $Q$ for any integer $T > 1$ and $\nu = \frac{\epsilon}{T}$. Let $\hat{\mu}_1(\nu)$ and $\hat{\mu}_2(\nu)$ be their empirical distributions. The following relations hold for any $\epsilon$,
$\epsilon'= \frac{\epsilon(T-1)}{T} , \epsilon{''}= \frac{\epsilon(T+1)}{T}$ and ${\bf d} =\Vert \cdot \Vert_\infty:$
\begin{align}
\PV(\hat{\mu}_1,\hat{\mu}_2, \epsilon'')
\leq
\PVh(S_1,S_2,\epsilon)
\leq
\PV(\hat{\mu}_1,\hat{\mu}_2,\epsilon') \label{PVineq1}\\
\PV(\mu_1,\mu_2,\epsilon'')
\leq
\PV(P,Q,\epsilon)
\leq
\PV(\mu_1,\mu_2,\epsilon'). \label{PVineq2}
\end{align}
\end{lemma}

We use the following structure of two discretizations.
\begin{definition}[Refinement of a discretization] \label{def_refinement}
Define an initial  $\epsilon$-discretization $C_1(\epsilon) = \{b_1,...,b_{N(\epsilon)}\}$ on $([0,1]^d, \Vert \cdot \Vert_{\infty})$. The refinement of the discretization, for any $\epsilon$ and $T>1$, is defined as a $\nu$-discretization on $C_2(\nu) = \{a_1,...,a_{N(\nu)}\}$, where $\nu=\epsilon/T$, such that each element of the refinement is a result of splitting an element of the initial cover to $(\epsilon/T)^d$ elements of equal volume.
\end{definition}

The next lemma bounds the difference between the PV on the discrete distributions $\hat{\mu}_1(\nu),\hat{\mu}_2(\nu)$ and the distributions $\mu_1(\nu)$ and $\mu_2(\nu)$.
\begin{lemma} \label{lem_discConvPV}
Let $C_1(\epsilon)$ be an $\epsilon$-discretization on $[0,1]^d$, and $C_2(\nu)$ its refined discretization (Definition \ref{def_refinement}). Let $\hat{\mu}_{i}(\epsilon)$ and $\mu_{i}(\epsilon)$ be distributions on $C_1(\epsilon)$, and
$\hat{\mu}_{i}(\nu)$ and $\mu_{i}(\nu)$  distributions on the refinement $C_2(\nu)$.  For any $\epsilon\in(0,1)$ and ${\bf d} =\Vert \cdot \Vert_\infty$ we have that
\begin{align*}
 &| \PV(\hat{\mu}_{1}(\nu),\hat{\mu}_{2}(\nu),\epsilon) -\PV(\mu_{1}(\nu),\mu_{2}(\nu),\epsilon) | \leq
   \frac{1}{2} \left( \Vert\mu_1(\epsilon)-\hat{\mu}_{1}(\epsilon) \Vert_1  + \Vert \mu_2(\epsilon)-\hat{\mu}_{2}(\epsilon) \Vert_1\right).
\end{align*}
\end{lemma}
Observe that the $L_1$-norm is computed over the elements of $C_1(\epsilon)$.

We use the following result provided by \cite{weissman2003inequalities} (Theorem 2.1).
\begin{lemma}
\label{Tsachy}
Let $\mu$ be a probability distribution on the set $\mathcal{A} = {1,...,a}$. Let $X =x_1,x_2,...,x_N$ be i.i.d. random variables distributed according to $\mu$, and $\hat{\mu}_N$ the resulting empirical distribution.
Then, for $\eta>0$
$$ \mathbb{P}(\Vert \mu-\hat{\mu}_N \Vert_1 \geq \eta) \leq (2^a-2)e^{-N\eta^2/2}.$$
\end{lemma}

\vspace{0.4cm}
\begin{proof}[Proof of Theorem \ref{Thm_PVtailsInftyDist}]

Set $\epsilon'= \frac{\epsilon(T-1)}{T}$ and $\epsilon{''}= \frac{\epsilon(T+1)}{T}$, and
define
$$m(T) =  PV(\mu_{1}(\nu),\mu_{2}(\nu),\epsilon') -  PV(\mu_{1}(\nu),\mu_{2}(\nu),\epsilon'').$$
By Lemma \ref{PVineq}, the value of $m(T)$ is positive.
Combining Lemma \ref{PVineq} with Lemma \ref{lem_discConvPV} yields
\begin{align} \label{eq1_proofThm1}
  & \PVh(S_1,S_2,\epsilon)  \leq  \PV(\hat{\mu}_1(\nu),\hat{\mu}_2(\nu), \epsilon') \\
                                 & \leq
                                 \PV(\mu_{1}(\nu),\mu_{2}(\nu),\epsilon') +
                                 \frac{1}{2} \Vert\mu_1(\epsilon')-\hat{\mu}_{1}(\epsilon') \Vert_1
                                 +\frac{1}{2}\Vert \mu_2(\epsilon')-\hat{\mu}_{2}(\epsilon') \Vert_1 \nonumber\\
                                 &=\PV(\mu_{1}(\nu),\mu_{2}(\nu),\epsilon'') +m(T) 
+\frac{1}{2} \Vert\mu_1(\epsilon'')-\hat{\mu}_{1}(\epsilon'') \Vert_1+\frac{1}{2}\Vert \mu_2(\epsilon'')-\hat{\mu}_{2}(\epsilon'') \Vert_1\nonumber\\
                                 & \leq \PV(P,Q,\epsilon) +m(T)+
                                 \frac{1}{2} \Vert\mu_1(\epsilon')-\hat{\mu}_{1}(\epsilon') \Vert_1 
                                 +\frac{1}{2}\Vert \mu_2(\epsilon')-\hat{\mu}_{2}(\epsilon') \Vert_1.\nonumber
\end{align}

Recall that the number of elements for an $\epsilon$-discretization on $C_1(\epsilon)$ is $\mathcal{N}(\epsilon)=(1/\epsilon)^d$.
Apply Lemma \ref{Tsachy} to $\Vert\mu_1(\epsilon')-\hat{\mu}_{1}(\epsilon') \Vert_1\leq \eta$ and $\Vert\mu_2(\epsilon')-\hat{\mu}_{2}(\epsilon') \Vert_1\leq \eta$ and combine the result with (\ref{eq1_proofThm1}) using the union bound. We have that with probability at least  $1- 2(2^{(1/\epsilon')^d}-2)e^{-N\eta^2/2}$
\begin{align}\label{eq3_proofThm1}
 \PVh(S_1,S_2,\epsilon)-  \PV(P,Q,\epsilon) \leq m(T)+\eta .
\end{align}
In a similar manner we have
\begin{align}\label{eq4_proofThm1}
  &\PVh(S_1,S_2,\epsilon)   \geq  \PV(\hat{\mu}_1(\nu),\hat{\mu}_2(\nu),\epsilon{''})  \\
                                 & \geq \PV(\mu_{1}(\nu),\mu_{2}(\nu),\epsilon{''}) -
                                 \frac{1}{2} \Vert\mu_1(\epsilon{''})-\hat{\mu}_{1}(\epsilon{''}) \Vert_1
                                 -\frac{1}{2}\Vert \mu_2(\epsilon{''})-\hat{\mu}_{2}(\epsilon{''}) \Vert_1\nonumber\\
                                 &= \PV(\mu_{1}(\nu),\mu_{2}(\nu),\epsilon{'})-m(T) 
                                 -\frac{1}{2} \Vert\mu_1(\epsilon{''})-\hat{\mu}_{1}(\epsilon{''}) \Vert_1
                                 -\frac{1}{2}\Vert \mu_2(\epsilon{''})-\hat{\mu}_{2}(\epsilon{''}) \Vert_1 \nonumber\\
                                 &\geq \PV(P,Q,\epsilon) -m(T)
                                  -\frac{1}{2} \Vert\mu_1(\epsilon{''})-\hat{\mu}_{1}(\epsilon{''}) \Vert_1
                                  -\frac{1}{2} \Vert \mu_2(\epsilon{''})-\hat{\mu}_{2}(\epsilon{''}) \Vert_1.\nonumber
\end{align}
Combining the result with the tail bounds of  $\hat{\mu}_{1}$, $\hat{\mu}_{2}$ from Lemma \ref{Tsachy}, and applying the union bound, we have that with probability at least  $1- 2(2^{(1/\epsilon'')^d}-2)e^{-N\eta^2/2}$
 \begin{align}\label{eq5_proofThm1}
\PV(P,Q,\epsilon)- \PVh(S_1,S_2,\epsilon)  \leq m(T)+\eta.
\end{align}

For $T \gg \epsilon$ we have that $\epsilon' \approx \epsilon''=\epsilon$ , and therefore the value of $m(T)\rightarrow 0$ as $T\rightarrow\infty$.
Taking $T\rightarrow\infty$ in (\ref{eq3_proofThm1}) and (\ref{eq5_proofThm1}), and combining the result
we get that for any $\delta\in(0,1)$ and $\eta = \sqrt{\frac{2(\log(2(2^{(1/\epsilon)^d}-2))+\log(1/\delta))}{N}}.$
\begin{align*}
\mathbb{P} \left( \left|\PVh(S_1,S_2,\epsilon)-  \PV(P,Q,\epsilon)\right| \geq \eta \right) \leq \delta.
\end{align*}
\end{proof}

\subsection*{Proofs of Lemmas \ref{PVineq},\ref{lem_discConvPV}}

\begin{proof}{ Lemma \ref{PVineq}}

Let sample $x_i\in S_1$ belong to the element $a_k$ in the $\nu$-discretization, that is $x_i\in B(a_k,\Vert\cdot\Vert_\infty,\nu=\frac{\epsilon}{T})$.
Recall that the $\epsilon$-neighborhood of a sample $x_i$ is the set $\text{ng}(x_i,\epsilon)=\{z: d(x_i,z)\leq \epsilon\}$, and the $\frac{\epsilon(T+1)}{T}$-neighborhood of bin $a_k$ is the set $\text{ng}(a_k,\frac{\epsilon(T+1)}{T})=\{z: d(a_k,z)\leq \frac{\epsilon(T+1)}{T}\}$.
For the left side of (\ref{PVineq1}), observe that for any such $x_i$ there exists values of $z$ such that $ \Vert z-a_k \Vert_{\infty} \leq \frac{\epsilon(T+1)}{T}$ but $\Vert z-x_i \Vert_{\infty} > \epsilon$, while for any $z$ for which $\Vert z-x_i \Vert_{\infty} \leq \epsilon$
also $ \Vert z-a_k \Vert_{\infty} \leq \frac{\epsilon(T+1)}{T}$.
As a result,
$\text{ng}(x_i,\epsilon) \subseteq \text{ng}(a_k,\frac{\epsilon(T+1)}{T})$.
Enlarging the number of neighbors adds edges to the bipartite graph describing the problem, and accordingly, a matching with a larger cardinality may be found.
In such a case, the number of unmatched samples is decreased, and therefore the PV is decreased, as it is the normalized sum of the unmatched samples.

For the right hand side of (\ref{PVineq1}), observe that when the discretization is $\frac{\epsilon(T-1)}{T}$, for any point $x_i\in B(a_k,\Vert\cdot\Vert_\infty,\nu)$ we have that $\text{ng}(x_i,\epsilon) \supseteq \text{ng}(a_k,\frac{\epsilon(T-1)}{T}),$
as the $\epsilon$-neighborhood of each point mass encloses the $\frac{\epsilon(T-1)}{T}$-neighborhood of its ascribed bin. As a result, the PV between the histograms $\hat{\mu}_1$ and $\hat{\mu}_2$ may correspond to a graph that has less edges, which may result in a maximum matching with a smaller cardinality. As a result, the discrete version may have a larger PV.\\
Inequalities (\ref{PVineq2}) hold, as the same claims apply for the discretization of the distributions.
\end{proof}

The following representation of Problem (\ref{PVdis}) will be useful for the proof of Lemma \ref{lem_discConvPV}.
\begin{lemma}\label{lem_PVeqiv}
The solution of Problem (\ref{PVdis}) may be obtained by solving the following problem
 \begin{align}\label{PVdisEqiv}
& \min _{ w_i,v_i,Z_{ij}}  \frac{1}{2}\sum_{i=1}^{N}|w_i| + \frac{1}{2}\sum_{j=1}^{N}|v_j| \\
& \sum_{a_j\in \text{ng}(a_i,\epsilon)} Z_{ij}+w_i = \mu_1(a_i), \,\, i = 1,...,N \nonumber\\
& \sum_{a_i\in \text{ng}(a_j,\epsilon)} Z_{ij}+v_j = \mu_2(a_j), \,\, j = 1,...,N \nonumber\\
& \,\,Z_{ij} \geq 0 ,\,\, \forall i,j, \nonumber
\end{align}
 which we call $\PV_{eq}(\mu_{1}(\nu),\mu_{2}(\nu),\epsilon).$
\end{lemma}
The lemma states that the constraints $w_i\geq 0 ,\,\, v_j\geq 0$ may be removed, and instead the sum in the objective is taken over the absolute values.
\begin{proof}{ Lemma \ref{lem_PVeqiv}}

First note that any solution of Problem (\ref{PVdis}) is a feasible solution of Problem (\ref{PVdisEqiv}), and so we have that the optimum  $\PV(\mu_1(\nu),\mu_2(\nu),\epsilon) \geq \PV_{eq}(\mu_{1}(\nu),\mu_{2}(\nu),\epsilon)$. We construct a solution of (\ref{PVdis}) that realizes the equality, and therefore is optimal. Namely, to show the problems are equivalent it is sufficient to show that any solution of (\ref{PVdisEqiv}) has a corresponding solution of (\ref{PVdis}) with the same objective value.

Let $w_i,v_j,Z_{ij}$ be the solution to (\ref{PVdisEqiv}). In the following, we construct a feasible solution $\tilde{w}_i,\tilde{v}_i,\tilde{Z}_{ij}$ to (\ref{PVdis}):
\begin{align*}
& \text{ \textbf{If} $w_i<0$ and $v_i>0$ set $\Delta_i=|w_i|$ and }\\
& \tilde{w}_i = w_i+\Delta_i=0, \,\, \tilde{v}_i = v_i+\Delta_i>0, 
\sum_{a_j\in \text{ng}(a_i)}\tilde{Z}_{ij} =\sum_{a_j\in \text{ng}(a_i)}Z_{ij}-\Delta_i.\\
& \text{ \textbf{If} $v_i<0$ and $w_i>0$ set $\Gamma_j=|v_j|$ and}\\
& \tilde{v}_i = v_i+\Gamma_i=0, \,\, \tilde{w}_i = w_i+\Gamma_i>0, 
 \sum_{a_j\in \text{ng}(a_i)}\tilde{Z}_{ji} =\sum_{a_j\in \text{ng}(a_i)}Z_{ji}-\Gamma_i.\\
& \text{ \textbf{If} both $w_i<0$ and $v_i<0$ set} \\
&\tilde{w}_i = w_i+\Delta_i+\Gamma_i>0, \,\, \tilde{v}_i = v_i+\Delta_i+\Gamma_i>0,
\sum_{a_j\in \text{ng}(a_i)}(\tilde{Z}_{ij} +\tilde{Z}_{ji}) =\sum_{a_j\in \text{ng}(a_i)}(Z_{ij}+Z_{ji})-\Delta_i-\Gamma_i.\\
& \text{\textbf{Otherwise},  set $\tilde{w}_i=w_i$, $\tilde{v}_j=v_j$, and $\tilde{Z}_{ij} =Z_{ij}.$}
\end{align*}
The resulting  $\tilde{w}_i,\tilde{v}_j,\tilde{Z}_{ij}$ obey the equality constraints in (\ref{PVdis}) while fixing $\tilde{w}_i\geq 0 ,\,\, \tilde{v}_j\geq 0$. It is easy to show that there exists $\tilde{Z}_{ij} \geq 0$ that obeys the equalities above.
The objective value of (\ref{PVdisEqiv}) with $w_i,v_j,Z_{ij}$ and of (\ref{PVdis}) with $\tilde{w}_i,\tilde{v}_j,\tilde{Z}_{ij}$ is equal:
\begin{align*}
& \sum_{i=1}^{N}\tilde{w_i} +\sum_{j=1}^{N}\tilde{v}_j =
\sum_{i=1}^{N}(w_i+v_i) 1_{[w_i\geq0 \, , \, v_i\geq 0 ]} + 
\sum_{i=1}^{N}((w_i+\Delta_i)+v_i+\Delta_i) 1_{[w_i< 0 \, , \, v_i\geq 0 ]}+\\
&
\sum_{j=1}^{N}(w_j+\Gamma_j +(v_j+\Gamma_j)) 1_{[w_j\geq 0 \, , \, v_j< 0 ]}+
 \sum_{i=1}^{N}((w_i+\Delta_i+\Gamma_i) +(v_i+\Gamma_i+\Delta_i) ) 1_{[w_i< 0 \, , \, v_i< 0 ]}\\
& = \sum_{i=1}^{N}|w_i| +\sum_{j=1}^{N}|v_j|.
\end{align*}

We conclude that $\tilde{w}_i,\tilde{v}_j,\tilde{Z}_{ij}$ attains  the optimal solution to Problem (\ref{PVdis}).
\end{proof}

\begin{proof}{ Lemma \ref{lem_discConvPV}}

Let $Z^*_{ij},w^*_i,v^*_j$ be the optimal arguments for which $\PV(\mu_{1},\mu_{2},\epsilon)$ is obtained (Problem (\ref{PVdis})). There are two stages to bounding the difference between $\PV(\mu_{1},\mu_{2},\epsilon)$ and $\PV(\hat{\mu}_{1},\hat{\mu}_{2},\epsilon)$.
First, by Lemma \ref{lem_PVeqiv} we know that given a solution $\PV_{eq}(\hat{\mu}_{1},\hat{\mu}_{2},\epsilon)$ we can find an equivalent solution $\PV(\hat{\mu}_{1},\hat{\mu}_{2},\epsilon)$. As a result, we may bound the difference between $\PV(\mu_{1},\mu_{2},\epsilon)$ and $\PV_{eq}(\hat{\mu}_{1},\hat{\mu}_{2},\epsilon)$ instead of the difference between $\PV(\mu_{1},\mu_{2},\epsilon)$ and $\PV(\hat{\mu}_{1},\hat{\mu}_{2},\epsilon)$.
To bound this difference, we change the solution $Z^*_{ij},w^*_i,v^*_j$ to describe a feasible solution to Problem (\ref{PVdisEqiv}) for distributions $\hat{\mu}_{1}$ and $\hat{\mu}_{2}$.

To obtain a feasible solution to Problem (\ref{PVdisEqiv}), we must fix the violations that are made to its constraints by substituting $Z^*_{ij},w^*_i,v^*_j$ into Problem (\ref{PVdisEqiv}).
The constraints are fixed in two manners. Some constraints are fixed by optimizing the transportation plan, described by matrix $Z$, within the refinement of the discretization. Additional violations are fixed by changing the variables $w_j$ and $v_j$.

Define $s_k = \{ a_i : \,  a_i\in B(b_k, \Vert \cdot \Vert _{\infty},\epsilon)  \}$; i.e., the set of bins $a_i\in C_2(\nu)$ that are a refinement of element $b_k\in C_1(\epsilon)$ (Definition \ref{def_refinement}). Let $|s_k|$ be the cardinality of this set. By definition, all the bins in $s_k$ are $\epsilon$-neighbors: $\forall a_i\in s_k , \, s_k \in \text{ng}(a_i,\epsilon)$.

For any $a_i,a_j\in s_k$, consider the following feasibility problem:
\begin{align}
& \text{Find } \,  C_{ij} \label{Cij_Feasibility}\\
 & \sum_{a_j\in  s_k } C_{ij} = c_i, \quad \forall a_i\in s_k, \nonumber\\
 & \sum_{a_i\in  s_k } C_{ij} = b_j,  \quad \forall a_j\in s_k, \nonumber\\
 & Z_{ij}^* +C_{ij} \geq 0,  \quad\quad \forall a_i,a_j\in s_k, \nonumber
\end{align}
where
\begin{align*}
& c_i \doteq (\hat{\mu}_1(a_i)-\mu_1(a_i))-\frac{1}{|s_k|}(\hat{\mu}_1(b_k)-\mu_1(b_k)),\\
& b_j \doteq (\hat{\mu}_2(a_j)-\mu_2(a_j))-\frac{1}{|s_k|}(\hat{\mu}_2(b_k)-\mu_2(b_k)).
\end{align*}
Note that $c_i$ and $b_i$ may be positive or negative, and that $\sum_{a_i\in s_k}c_i = 0$ and $\sum_{a_j\in s_k}b_j = 0$.

We show that the following values $\bar{w}_i,\bar{v}_j,\bar{Z}_{ij}$ for $i,j=1,...,N(\nu)$ are a feasible solution to Problem (\ref{PVdisEqiv}).
\begin{align} \label{feasPVeq}
 \bar{w}_i     &= w_{i}^*
                +\frac{1}{|s_k|}(\hat{\mu}_1(b_k)-\mu_1(b_k))
               \\
 \bar{v}_j     & = v_{j}^*
                +\frac{1}{|s_k|}(\hat{\mu}_2(b_k)-\mu_2(b_k))
\nonumber\\
\bar{Z}_{ij} &=
\begin{cases}
Z_{ij}^*         & \text{if $a_j\in s_k^c ,\, a_i\in s_k$,}
\nonumber\\
Z_{ij}^* +C_{ij} & \text{if $a_j\in s_k,   \, a_i\in s_k$,}
\end{cases}
\end{align}
where $C_{ij}$ is the solution to the (\ref{Cij_Feasibility}).

First, we show that Problem (\ref{Cij_Feasibility}) is feasible. To do so, we consider its dual representation.
Define $v=\text{Vec}(\{C_{ij}\}_{a_i,a_j\in s_k})\in \mathbb{R}^{|s_k|^2 \times 1}$, the vector form of the sub-matrix $\{C_{ij}\}_{a_i,a_j\in s_k}$.
Similarly, let $z^* = \text{Vec}(\{Z^*_{ij}\}_{a_i,a_j\in s_k})\in \mathbb{R}^{|s_k|^2\times 1}$. Let $A\in \mathbb{R}^{2|s_k|\times |s_k|^2}$ be the zero-one matrix defined by the left-hand sides of the equality constraints in (\ref{Cij_Feasibility}), and  $d = [c_1,...,c_{|s_k|},b_1,...,b_{|s_k|}]^T\in \mathbb{R}^{2|s_k|\times 1}$, the vector defined by the right-hand sides of these constraints.
Using these notations, Problem (\ref{Cij_Feasibility}) is equivalent to
\begin{align*}
& \text{Find } \,  v \\
& Av = d \, , \quad -v-z^* \leq 0,
\end{align*}
whose dual representation is the existence of  $\lambda\in \mathbb{R}^{|s_k|^2\times 1},\eta \in \mathbb{R}^{2|s_k|\times 1}$ for which
\begin{align} \label{Cij_Dual}
& g(\lambda,\eta) = \inf_{v} \lambda^T (-v-z^*)+ \eta^T(Av-d) >0,\\
& \lambda \geq 0. \nonumber
\end{align}
The value of $ g(\lambda,\eta)$ in (\ref{Cij_Dual}) is not $-\infty$ only when $A^T\eta - \lambda = 0$, for which
\begin{align*}
& g(\lambda,\eta) = \inf_{v} \lambda^T (-v-z^*)+ \eta^T(Av-d) = \\
& \inf_{v} v^T(-\lambda+A^T\eta)- \lambda^Tz^*-\eta^Td = -\lambda^Tz^*-\eta^Td.
\end{align*}
Since $z^*\geq 0$ and $\lambda \geq 0$, we have that $-\lambda^Tz^*\leq 0$. By noting that ${\bf 1}^Td = \sum_{a_i\in s_k}c_i+\sum_{a_j\in s_k}b_j=0$, we have that $-\eta^Td \leq -\min{\eta_\ell}\cdot{\bf1}^Td =0$. We conclude that  $g(\lambda,\eta)\leq 0$, and therefore Problem (\ref{Cij_Dual}) is infeasible. By the theorem of alternatives \cite{Boyd2004} Problem (\ref{Cij_Feasibility}) is feasible.

Next, we show that the proposed solution $\bar{Z}_{ij},\bar{w_i},\bar{v_j}$ is indeed a feasible solution of Problem (\ref{PVdisEqiv}).
The constraints $\bar{Z}_{ij}\geq 0$ hold by the feasibility of (\ref{Cij_Feasibility}).
The equality constraints also hold:
\begin{align*}
 &\sum_{a_j\in \text{ng}(a_i,\epsilon)} \bar{Z}_{ij} +\bar{w}_i =
   \sum_{a_j\in \text{ng}(a_i,\epsilon)} {Z}^*_{ij}
  + \sum_{a_j\in s_k}C_{ij}
  + \bar{w}_i=\\
  &\sum_{a_j\in \text{ng}(a_i,\epsilon)} {Z}^*_{ij} +c_i+\bar{w}_i=\sum_{a_j\in\text{ng}(a_i,\epsilon)}{Z}^*_{ij}
  +\hat{\mu}_1(a_i)-\mu_1(a_i)\\
  &-\frac{1}{|s_k|}(\hat{\mu}_1(b_k)-\mu_1(b_k))
    +w^*_i+\frac{1}{|s_k|}(\hat{\mu}_1(b_k)-\mu_1(b_k))\\
  & = \mu_1(a_i) +(\hat{\mu}_1(a_i)-\mu_1(a_i)) =\hat{\mu}_1(a_i),
 \end{align*}
and
\begin{align*}
 &\sum_{a_i\in \text{ng}(a_j,\epsilon)} \bar{Z}_{ij} +\bar{v}_j =
   \sum_{a_i\in \text{ng}(a_j,\epsilon)} {Z}^*_{ij}
  + \sum_{a_i\in s_k}C_{ij}
  + \bar{v}_j=\\
  & \sum_{a_i\in \text{ng}(a_j,\epsilon)} {Z}^*_{ij} +b_j+\bar{v}_j= \sum_{a_i\in\text{ng}(a_j,\epsilon)}{Z}^*_{ij}
   +\hat{\mu}_2(a_j)-\mu_2(a_j)\\
  &-\frac{1}{|s_k|}(\hat{\mu}_2(b_k)-\mu_2(b_k))
  +v^*_j +\frac{1}{|s_k|}(\hat{\mu}_2(b_k)-\mu_2(b_k))\\
  & = \mu_2(a_j) +(\hat{\mu}_2(a_j)-\mu_2(a_j)) =\hat{\mu}_2(a_j).
 \end{align*}

To conclude the proof, we bound the difference of the objective of Problem (\ref{PVdis}), obtained with the values  $Z^*_{ij},w^*_i,v^*_j$, and  the objective of Problem (\ref{PVdisEqiv}), obtained with the values $\bar{Z}_{ij},\bar{w_i},\bar{v_j}$.

Since the discretization defined on $C_1(\nu)$ is a refinement of $C_2(\epsilon)$ (Definition \ref{def_refinement}), we have that
\begin{align}\label{eq1diff}
\sum_{i=1}^{N(\nu)}\left(|\bar{w}_i|+|\bar{v}_i|\right)=\sum_{k=1}^{N(\epsilon)}\sum_{a_i\in s_k}( |\bar{w}_i| + |\bar{v}_i|).
\end{align}
Substituting the values of $|\bar{w}_i|,|\bar{v}_i|$ by their assignment in (\ref{feasPVeq}) we obtain
\begin{align}\label{eq2diff}
        &\frac{1}{2}\sum_{i=1}^{N(\nu)}( |\bar{w}_i| + |\bar{v}_i|)                     -\frac{1}{2}\sum_{i=1}^{N(\nu)}\left(w_i^*+ v_i^*\right)=\\
        &\frac{1}{2}\sum_{k=1}^{N(\epsilon)}\sum_{a_i\in s_k}( |\bar{w}_i| + |\bar{v}_i|)-\frac{1}{2}\sum_{i=1}^{N(\nu)}\left(w_i^*+ v_i^*\right)=
           \nonumber\\
             &\frac{1}{2}\sum_{k=1}^{N(\epsilon)}\sum_{a_i\in s_k} | w_i^*+\frac{1}{|s_k|}(\hat{\mu_1}(b_k)-\mu_1(b_k))|+ 
            \frac{1}{2}\sum_{k=1}^{N(\epsilon)}\sum_{a_i\in s_k} |v_i^*+\frac{1}{|s_k|}(\hat{\mu_2}(b_k)-\mu_2(b_k))|-
            \frac{1}{2}\sum_{i=1}^{N(\nu)}\left(w_i^*+ v_i^*\right).\nonumber
\end{align}
Applying the triangle inequality on each element in the sum:
 $$| w_i^*+\frac{1}{|s_k|}(\hat{\mu_1}(b_k)-\mu_1(b_k))|\leq | w_i^*|+\frac{1}{|s_k|}|\hat{\mu_1}(b_k)-\mu_1(b_k)|$$
 $$|v_i^*+\frac{1}{|s_k|}(\hat{\mu_2}(b_k)-\mu_2(b_k))|\leq |v_i^*|+\frac{1}{|s_k|}|\hat{\mu_2}(b_k)-\mu_2(b_k)|,$$
  as well as noting that $w_i^*,v_j^*\geq0$ by definition, we have that
\begin{align}\label{eq3diff}
 & \frac{1}{2}\sum_{i=1}^{N(\nu)}\left(|\bar{w}_i|+|\bar{v}_i|\right)-\frac{1}{2}\sum_{i=1}^{N(\nu)}\left(w_i^*+ v_i^*\right) \leq  
            \frac{1}{2} \Vert \hat{\mu}_1(\epsilon) - \mu_1(\epsilon) \Vert_1 + \frac{1}{2} \Vert \hat{\mu}_2(\epsilon) - \mu_2(\epsilon) \Vert_1.\nonumber
\end{align}

By Lemma \ref{lem_PVeqiv} we have that the solution of Problem (\ref{PVdis}) may be obtained by solving Problem (\ref{PVdisEqiv}). Therefore,
combining (\ref{eq3diff}) with Lemma \ref{lem_PVeqiv} we have that
\begin{align*}
& PV(\hat{\mu}_1(\nu),\hat{\mu}_2(\nu),\epsilon)-PV(\mu_1(\nu),\mu_2(\nu),\epsilon)= PV_{eq}(\hat{\mu}_1(\nu),\hat{\mu}_2(\nu),\epsilon)-PV(\mu_1(\nu),\mu_2(\nu),\epsilon)\leq  \\
 &\frac{1}{2}\sum_{i=1}^{N(\nu)}\left(|\bar{w}_i|+|\bar{v}_i|\right)-\frac{1}{2}\sum_{i=1}^{N(\nu)}\left(w_i^*+ v_i^*\right)\leq
\frac{1}{2} \Vert \hat{\mu}_1(\epsilon) - \mu_1(\epsilon) \Vert_1 + \frac{1}{2} \Vert \hat{\mu}_2(\epsilon) - \mu_2(\epsilon) \Vert_1.
\end{align*}
The first inequality holds as the solution $\bar{Z}_{ij},\bar{w_i},\bar{v_j}$ is a feasible solution of Problem (\ref{PVdisEqiv}), but may not be optimal.

Using an analogous procedure starting at the optimal solution  $PV(\hat{\mu}_1(\nu),\hat{\mu}_2(\nu),\epsilon)$ and finding a feasible solution for distributions $\mu_{1}(\nu),\mu_{2}(\nu)$ we obtain
\begin{align*}
PV(\mu_{1}(\nu),\mu_{2}(\nu),\epsilon) -PV(\hat{\mu}_{1}(\nu),\hat{\mu}_{2}(\nu),\epsilon) \leq   \Vert\mu_1(\epsilon)-\hat{\mu}_{1}(\epsilon) \frac{1}{2}\Vert_1+\frac{1}{2} \Vert \mu_2(\epsilon)-\hat{\mu}_{2}(\epsilon) \Vert_1.
\end{align*}
Combining the last two inequalities concludes the proof of Lemma \ref{lem_discConvPV}.
\end{proof}

\subsection{Proof of Theorem \ref{Thm_Spheres}}
We restate the theorem:
\begin{reptheorem}{Thm_Spheres}
 Let $P=Q$ be the uniform distribution on $\mathbb{S}^{d-1}$, a unit ($d-1$)--dimensional hyper-sphere.  Let $S_1 =\{x_1,...,x_N\}\sim P$ and $S_2 =\{y_1,...,y_N\}\sim Q$ be two i.i.d. samples. For any $\epsilon,\epsilon',\delta\in(0,1)$, $0\leq\eta<2/3$ and sample size
 $ \frac{\log(1/\delta)}{2(1-3\eta/2)^2}\leq N\leq \eta/2 e^{d(1-\frac{\epsilon^2}{2})/2},$
 we have $PV\brck{P,Q,\epsilon'}=0$ and
\begin{align}
\mathbb{P}(\PVh\brck{S_1,S_2,\epsilon} > \eta ) \geq  1-\delta.
\end{align}
\end{reptheorem}

\begin{proof}
We use the following definitions and lemmas.

\begin{definition}
The spherical cap of radius $r$ about a point $x$ is
$$C(r,x) = \left\{ z\in S^{d-1} \, : d(z,x) \leq r \right\}.$$
\end{definition}

\begin{lemma}
The spherical cap of radius $r$ about a point $x$ on a unit sphere is equal to
$$C(r,x) = \left\{ z\in S^{d-1} \, : <z,x> \, \geq \sqrt{1-\frac{r^2}{2}} \right\}.$$
\end{lemma}

\begin{lemma} \label{Lem_capBound}
Let $\eta=\sqrt{1-\frac{r^2}{2}}$. For $0\leq\eta<1$, the cap $C(r,x)$ on $S^d-1$ has a measure at most $e^{-d\eta^2/2}.$
\end{lemma}

Let $p =\mathbb{P}(\text{ng}_{S_2}(x)= \emptyset)$ be the probability of an empty neighbor set.
The next lemma bounds this probability.
\begin{lemma}\label{Lem_EmptyNg}
The probability of an empty neighbor set $\mathbb{P}(\text{ng}_{S_2}(x)= \emptyset)\geq 1-  Ne^{-d(1-\frac{\epsilon^2}{2})/2}.$
\end{lemma}
\begin{proof}
\begin{align*}
p = &\mathbb{P}( \text{ng}_{S_2}(x) =\emptyset )  =1 - \mathbb{P}( \text{ng}_{S_2}(x) \neq \emptyset )
                                                  =1 - \mathbb{P}(\exists y_j\in S_2 \,\, ;  y_j \in C(\epsilon,x_i))\\
                                    & \geq 1 - N\mathbb{P}(y \in C(\epsilon,x))
                                   \geq 1-  Ne^{-d(1-\frac{\epsilon^2}{2})/2},
\end{align*}
where the first inequality is due to the union bound, and the second by Lemma \ref{Lem_capBound}.
\end{proof}

We consider the probability that the $\widehat{PV}$  is grater than some $0\leq\eta <1$.
Note, that since $PV(P,Q)=0$ this is also the difference between the empirical and distributional PV.
Let $e = \{ x_i\in S_1 : \, \text{ng}_{S_2}(x_i)= \emptyset\}$ be the set of samples in $S_1$ without neighbors, and $N_e$ its cardinality.
\begin{align}\label{eq_cap1}
\mathbb{P}(\widehat{PV}(S_1,S_2,\epsilon)> \eta )  &\geq \mathbb{P}( \frac{N_e}{N} > \eta ) =  1- \mathbb{P}( N_e \leq N\eta )\geq 1- \mathbb{P}( N_e \leq \lceil N\eta \rceil)\\
&=1 - \sum_{i=0}^{\lceil N\eta \rceil} {N\choose i}(p)^i(1-p)^{N-i}.\nonumber
\end{align}
The first inequality holds, as $\widehat{PV}(S_1,S_2,\epsilon)> \eta$ is obtained when $N_e>\eta N$ samples from $S_1$ have no neighbors from $S_2$ in their $\epsilon$-neighborhood. Note that since $n=m$ there are also exactly $N_e$ sample from $S_2$ which are not matched.

By Chernoff's inequality we have that
\begin{align}\label{eq_Hoeff}
\sum_{i=0}^{\lceil N\eta \rceil} {n\choose i}(1-p)^ip^{N-i}\leq \exp(-2N(p-\eta)^2).
\end{align}

Combining Equations (\ref{eq_cap1}) and (\ref{eq_Hoeff}) we get
\begin{align}\label{eq_cap2}
&\mathbb{P}(\widehat{PV}(S_1,S_2,\epsilon)> \eta )  \geq 1 - \exp(-2N(p-\eta)^2) .
\end{align}

By Lemma \ref{Lem_EmptyNg}, we have that $p\geq 1-  Ne^{-d(1-\frac{\epsilon^2}{2})/2}$.

If $0\leq\eta<2/3$ and $Ne^{-d(1-\frac{\epsilon^2}{2})/2}<\eta/2 $, we have that
$$p-\eta \geq 1- Ne^{-d(1-\frac{\epsilon^2}{2})/2}-\eta>1-3\eta/2>0.$$
Substituting the last inequality to (\ref{eq_cap2}):
$$ \mathbb{P}(\widehat{PV}(S_1,S_2,\epsilon)> \eta ) \geq 1 -\exp(-2N(1-3\eta/2)^2).$$
The theorem statement is obtained for any $N,d$ and $\eta$ for which  $2N(1-3\eta/2)^2 \geq \log(\frac{1}{\delta})$.
\end{proof}

\end{document}